\documentclass{article}

\usepackage[final]{neurips_2021}

\usepackage[final]{showkeys}

\usepackage[utf8]{inputenc} 
\usepackage[T1]{fontenc}    
\usepackage{hyperref}       
\usepackage{url}            
\usepackage{booktabs}       
\usepackage{amsfonts}       
\usepackage{nicefrac}       
\usepackage{microtype}      
\usepackage{xcolor}         
\usepackage{graphicx}
\usepackage{subcaption}
\usepackage{url}
\usepackage{amsthm}
\usepackage{amssymb}
\usepackage{amsmath}

\newcommand{\wtilde}{\widetilde}
\newcommand{\Id}{\operatorname{Id}}
\newcommand{\RR}{\mathbb{R}}

\newcommand{\cD}{\mathcal{D}}
\renewcommand{\P}{\mathbb{P}}

\newcommand{\cT}{\mathsf{T}}

\newcommand{\cB}{\mathcal{B}}

\newcommand{\cN}{\mathcal{N}}
\newcommand{\colspan}{\text{colspan}}
\newcommand{\cL}{\mathcal{L}}

\newcommand{\diag}{\operatorname{diag}}

\newcommand{\Sym}{\text{Sym}}
\newcommand{\bI}{\mathbf{1}}

\newcommand{\PP}{\mathbb{P}}

\newcommand{\lrr}[1]{\left(#1\right)}
\newcommand{\gives}{\to}

\newcommand{\set}[1]{\{#1\}}
\newcommand{\oL}{\overline{\mathcal{L}}}
\newcommand{\wPi}{\wtilde{\Pi}}
\newcommand{\E}{\mathbb{E}}
\newcommand{\Tr}{\text{tr}}
\newcommand{\Var}{\text{Var}}
\newcommand{\norm}[1]{\left\|#1\right\|}
\newcommand{\Ee}[1]{\E\left[#1\right]}
\newcommand{\R}{\mathbb{R}}
\newcommand{\eps}{\varepsilon}
\newcommand{\tr}{\text{tr}}
\newcommand{\twiddle}[1]{\wtilde{#1}}

\newtheorem{theorem}{Theorem}[section]
\newtheorem{lemma}[theorem]{Lemma}
\newtheorem{prop}[theorem]{Proposition}
\newtheorem{corr}[theorem]{Corollary}
\newtheorem{remark}[theorem]{Remark}

\newtheorem{definition}[theorem]{Definition}

\numberwithin{equation}{section}
\numberwithin{figure}{section}

\title{How Data Augmentation affects Optimization for Linear Regression}

\author{%
  Boris Hanin$^*$ \\
Department of Operations Research\\ and Financial Engineering\\
Princeton University \\
\texttt{bhanin@princeton.edu}
\And Yi Sun\thanks{Equal contribution}\\
Department of Statistics\\
University of Chicago \\
\texttt{yisun@statistics.uchicago.edu}
}

\begin{document}

\maketitle

\begin{abstract}
Though data augmentation has rapidly emerged as a key tool for optimization
in modern machine learning, a clear picture of how augmentation schedules affect
optimization and interact with optimization hyperparameters such as learning
rate is nascent. In the spirit of classical convex optimization
and recent work on implicit bias, the
present work analyzes the effect of augmentation on optimization in
the simple convex setting of linear regression with MSE loss.

We find joint schedules for learning rate and data
augmentation scheme under which augmented gradient descent provably
converges and characterize the resulting minimum. Our results
apply to arbitrary augmentation schemes, revealing complex interactions between learning
rates and augmentations even in the convex setting. Our approach interprets
augmented (S)GD as a stochastic optimization
method for a time-varying sequence of proxy losses. This gives a
unified way to analyze learning rate, batch size, and augmentations
ranging from additive noise to random projections. From this perspective,
our results, which also give rates of convergence, can be viewed as
Monro-Robbins type conditions for augmented (S)GD. 
\end{abstract}

\section{Introduction} \label{sec:intro}

Data augmentation, a popular set of techniques in which data is augmented
(i.e. modified) at every optimization step, has become increasingly crucial
to training models using gradient-based optimization.  However, in modern overparametrized
settings where there are many different minimizers of the training loss, the specific
minimizer selected by training and the quality of the resulting model
can be highly sensitive to choices of augmentation hyperparameters.  As a result, practitioners
use methods ranging from simple grid search to Bayesian optimization and reinforcement
learning \citep{cubuk2019autoaugment, cubuk2020randaugment, ho2019population} 
to select and schedule augmentations by changing hyperparameters over the course of
optimization. Such approaches, while effective, often require
extensive compute and lack theoretical grounding.

These empirical practices stand in contrast to theoretical results from  the
implicit bias and stochastic optimization literature.  The extensive recent literature on
implicit bias \citep{gunasekar2020characterizing, soudry2018implicit, wu2021direction}
gives provable guarantees on which minimizer of the training loss is selected by GD and SGD
in simple settings, but considers cases without complex scheduling.  On the other
hand, classical theorems in stochastic optimization, building on the Monro-Robbins theorem in
\citep{robbins1951stochastic}, give provably optimal learning rate schedules for
strongly convex objectives.  However, neither line of work addresses
the myriad augmentation and hyperparameter choices crucial to gradient-based training
effective in practice.

The present work takes a step towards bridging this gap.  We consider two main questions
for a learning rate schedule and data augmentation policy:
\begin{enumerate}
  \item[1.] When and at what rate will optimization converge?
  \item[2.] Assuming optimization converges, what point does it converge to?
\end{enumerate}
To isolate the effect \textit{on optimization} of jointly scheduling learning rate
and data augmentation schemes, we consider these questions in the simple convex
model of linear regression with MSE loss:
\begin{equation}\label{E:linear-model}
\mathcal L(W; \cD) = \frac{1}{N} \norm{Y - WX}_F^2.
\end{equation}
In this setting, we analyze the effect of the data augmentation policy on the
optimization trajectory $W_t$
of augmented (stochastic) gradient descent\footnote{Both GD and SGD fall into our framework.
  To implement SGD, we take $\cD_t$ to be a subset of $\cD$.}, which follows the update equation 
\begin{equation}\label{E:aug-update}
W_{t+1} = W_t - \eta_t \nabla_W\cL(W; \cD_t)\big|_{W=W_t}.
\end{equation}
Here, the dataset $\cD=\lrr{X,Y}$ contains $N$ datapoints arranged into data matrices
$X \in \RR^{n \times N}$ and $Y \in \RR^{p \times N}$ whose columns consist
of inputs $x_i \in \RR^n$ and outputs $y_i \in \RR^p$. In this context, we take a flexible
definition \textit{data augmentation scheme} as any procedure that consists, at every
step of optimization, of replacing the dataset $\cD$ by a randomly augmented variant
which we denote by $\cD_t=\lrr{X_t,Y_t}$.  This framework is flexible enough to handle SGD and
commonly used augmentations such as additive noise \cite{grandvalet1997noise},
CutOut \cite{devries2017improved}, SpecAugment \cite{park2019specaugment}, Mixup \citep{zhang2017mixup},
and label-preserving transformations (e.g. color jitter,
geometric transformations \citep{simard2003best})).

We give a general answer to Questions 1 and 2 for arbitrary data augmentation schemes.
Our main result (Theorem \ref{thm:mr-main}) gives sufficient conditions for optimization
to converge in terms of the learning rate schedule and simple $2^\text{nd}$ and
$4^\text{th}$ order moment statistics of augmented data matrices.  When convergence
occurs, we explicitly characterize the resulting optimum in terms of these statistics.
We then specialize our results to (S)GD with modern augmentations such as additive
noise \citep{grandvalet1997noise} and random projections (e.g. CutOut \citep{devries2017improved}
and SpecAugment \cite{park2019specaugment}).
In these cases, we find learning rate and augmentation parameters which ensure
convergence with rates to the minimum norm optimum for overparametrized linear regression.
To sum up, our main contributions are:
\begin{itemize}
\item[1.] We analyze \textit{arbitrary} data augmentation schemes for linear regression
  with MSE loss, obtaining explicit sufficient conditions on the joint schedule of the data
  augmentation policy and the learning rate for (stochastic) gradient descent that guarantee
  convergence with rates in Theorems \ref{thm:mr-main} and \ref{thm:mr-quant-informal}.
  The resulting augmentation-dependent optimum encodes the ultimate effect of augmentation
  on optimization, and we characterize it in Theorem \ref{thm:mr-main}.  Our results
  generalize Monro-Robbins theorems \cite{robbins1951stochastic} to situations where the
  stochastic optimization objective may change at each step.

\item[2.]  We specialize our results to (stochastic) gradient descent with additive input
  noise (\S \ref{sec:add-noise}) or random projections of the input (\S \ref{sec:projections}),
  a proxy for the popular CutOut and SpecAugment augmentations \cite{devries2017improved, park2019specaugment}.
  In each case,
  we find that jointly scheduling learning rate and augmentation strength is critical for
  allowing convergence with rates to the minimum norm optimizer. We find specific
  schedule choices which guarantee this convergence with rates (Theorems \ref{thm:gauss-gd},
  \ref{thm:gauss-sgd}, and \ref{thm:proj-main}) and validate our results empirically
  (Figure \ref{fig:figure}).  This suggests explicitly adding learning rate schedules
  to the search space for learned augmentations as in \cite{cubuk2019autoaugment, cubuk2020randaugment},
  which we leave to future work. 
\end{itemize}

\section{Related Work}\label{S:lit-rev}

In addition to the extensive empirical work on data augmentation cited elsewhere
in this article, we briefly catalog other theoretical work on data augmentation
and learning rate schedules. The latter were first considered in the seminal work
\cite{robbins1951stochastic}. This spawned a vast literature on \textit{rates} of
convergence for GD, SGD, and their variants. We mention only the relatively recent
articles \cite{bach2013non, defossez2015averaged,bottou2018optimization, smith2018don, ma2018power}
and the references therein. The last of these, namely \cite{ma2018power}, finds optimal
choices of learning rate and batch size for SGD in the overparametrized linear setting.

A number of articles have also pointed out in various regimes that data augmentation
and more general transformations such as feature dropout correspond in part to
$\ell_2$-type regularization on model parameters, features, gradients, and Hessians.
The first article of this kind of which we are aware is
\cite{bishop1995training}, which treats the case of additive Gaussian noise
(see \S \ref{sec:add-noise}). More recent work in this direction includes
\cite{chapelle2001vicinal, wager2013dropout, lejeune2019implicit, liu2020penalty}. There are also several
articles investigating \textit{optimal} choices of $\ell_2$-regularization for linear
models (cf e.g. \cite{wu2018sgd, wu2020optimal, bartlett2020benign}). These articles
focus directly on the generalization effects of ridge-regularized minima but not on
the dynamics of optimization. We also point the reader to \cite{lewkowycz2020training},
which considers optimal choices for the weight decay coefficient empirically in neural
networks and analytically in simple models. 

We also refer the reader to a number of recent attempts to characterize the benefits
of data augmentation. In \cite{rajput2019does}, for example, the authors quantify
how much augmented data, produced via additive noise, is needed to learn positive margin
classifiers. \cite{chen2019invariance}, in contrast, focuses on the case of
data invariant under the action of a group. Using the group action to generate label-preserving
augmentations, the authors prove that the variance of any function depending only
on the trained model will decrease. This applies in particular to estimators for
the trainable parameters themselves. \cite{dao2019kernel} shows augmented
$k$-NN classification reduces to a kernel method for augmentations transforming
each datapoint to a finite orbit of possibilities.  It also gives a second order
expansion for the proxy loss of a kernel method under such augmentations
and interprets how each term affects generalization. Finally, the article
\cite{wu2020generalization} considers both label preserving and noising augmentations,
pointing out the conceptually distinct roles such augmentations play.

\section{Time-varying Monro-Robbins for linear models under augmentation} \label{sec:mr}

We seek to isolate the impact of data augmentation on optimization in
the simple setting of augmented (stochastic) gradient descent for linear
regression with the MSE loss \eqref{E:linear-model}. Since the augmented
dataset $\cD_t$ at time $t$ is a stochastic function of $\cD$, we may view
the update rule \eqref{E:aug-update} as a form of stochastic optimization for the
\textit{proxy loss at time $t$}
\begin{equation}\label{eq:proxy-opt}
\overline{\cL}_t(W) := \E_{\cD_t}\left[\cL(W; \cD_t)\right]
\end{equation}
which uses an unbiased estimate of the gradient of $\cL(W; \cD_t)$ from a
single draw of $\cD_t$.  The connection between data augmentation and this proxy loss
was introduced in \cite{bishop1995training, chapelle2001vicinal}, but we
now consider it in the context of stochastic optimization.
In particular, we
consider scheduling the augmentation, which allows the distribution of $\cD_t$
to change with $t$ and thus enables optimization to converge to points which
are not minimizers of the proxy loss $\overline{\cL}_t(W)$ at any fixed time.

Our main results, Theorems \ref{thm:mr-main} and \ref{thm:mr-quant-informal},
 provide sufficient conditions for jointly
scheduling learning rates and general augmentation schemes to guarantee
convergence of augmented gradient descent in the linear regression
model \eqref{E:linear-model}.  Before stating them, we first give examples
of augmentations falling into our framework, which we analyze
using our general results in \S \ref{sec:add-noise} and \S \ref{sec:projections}.
\begin{itemize}
\item \textbf{Additive Gaussian noise:} For SGD with batch
size $B_t$ and noise level $\sigma_t > 0$, this corresponds to 
$X_t = c_t(XA_t + \sigma_t \cdot G_t)$ and $Y_t = c_t YA_t$, 
where $G_t$ is a matrix of i.i.d. standard Gaussians, $A_t \in \R^{N \times B_t}$
has i.i.d. columns with a single non-zero entry equal to $1$ chosen uniformly
at random and $c_t = \sqrt{N/B_t}$ is a normalizing factor. The proxy loss is
\begin{equation}\label{E:gauss-proxy}
    \overline{\cL}_t(W) = \cL(W;\cD) + \sigma_t^2 \norm{W}_F^2, 
\end{equation}
which adds an $\ell_2$ penalty.  We analyze this case in \S \ref{sec:add-noise}. 
    
\item \textbf{Random projection:} This corresponds to $X_t = \Pi_t X$
and $Y_t = Y$, where $\Pi_t$ is an orthogonal projection onto a random subspace.
For $\gamma_t = \Tr(\Pi_t)/n$, the proxy loss is
\[
\overline{\cL}_t(W) = \frac{1}{N} \|Y - \gamma_t W X\|_F^2
+ \frac{1}{N}\gamma_t (1 - \gamma_t) \frac{1}{n} \|X\|_F^2 \cdot \|W\|_F^2
+ O(n^{-1}),
\]
adding a data-dependent $\ell_2$ penalty and applying Stein-type shrinkage
 on input data. We analyze this in \S \ref{sec:projections}. 
\end{itemize}

In addition to these augmentations, the augmentations below also fit
into our framework, and Theorems \ref{thm:mr-main} and
\ref{thm:mr-quant-informal} apply. However, in these cases, explicitly
characterizing the learned minimum beyond the general description
given in Theorems \ref{thm:mr-main} and \ref{thm:mr-quant-informal} is
more difficult, and we thus leave interpretion of these
specializations to future work.

\begin{itemize}
\item \textbf{Label-preserving transformations:} For a 2-D image viewed
as a vector $x \in \R^n$, geometric transforms (with pixel interpolation)
or other label-preserving transforms such as color jitter take the form
of linear transforms $\R^n \to \R^n$. We may implement such augmentations
in our framework by $X_t = A_t X$ and $Y_t = Y$ for some random transform
matrix $A_t$.

\item \textbf{Mixup:} To implement Mixup, we can take $X_t = X A_t$ and
$Y_t = Y A_t$, where $A_t \in \R^{N \times B_t}$ has i.i.d. columns containing
two random non-zero entries equal to $1 - c_t$ and $c_t$ with mixing
coefficient $c_t$ drawn independently from a $\text{Beta}(\alpha_t, \alpha_t)$ distribution
for a parameter $\alpha_t$.
\end{itemize}

\subsection{A general time-varying Monro-Robbins theorem}\label{S:mr}

\begin{figure*}[t!] \centering
\begin{subfigure}[t]{0.49\textwidth} \centering
  \includegraphics[width=\textwidth]{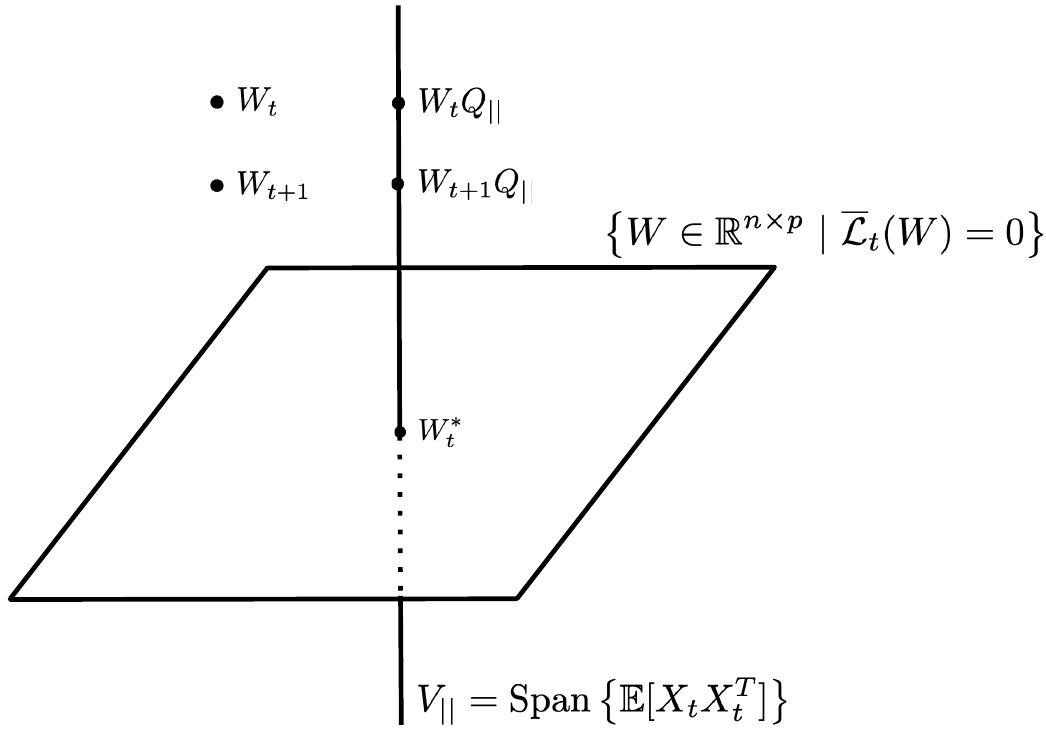}
  \caption{The time $t$ proxy loss $\overline{\mathcal L}_t$ is non-degenerate on $V_{\parallel}$ with  minimal norm minimizer $W_t^*$. The increment $W_{t+1}-W_{t}$ is in $V_{\parallel}$, so only  the projection $W_tQ_{\parallel}$ of $W_t$ onto $V_{\parallel}$ changes.}
  \label{fig:opt}
\end{subfigure}
\hfill
\begin{subfigure}[t]{0.49\textwidth} \centering
  \includegraphics[width=\textwidth]{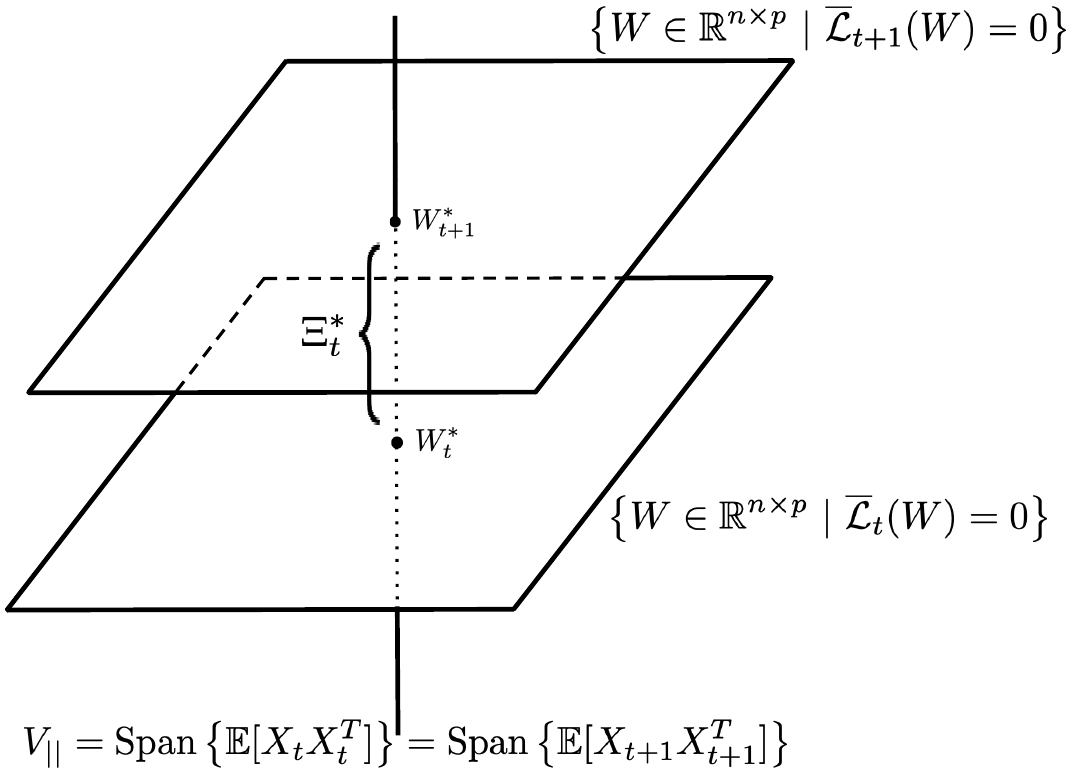}
  \caption{The proxy losses $\overline{\mathcal L}_t$ and $\overline{\mathcal L}_{t+1}$ at consecutive times will generally have different minimal norm minimizers $W_t^*,\, W_{t+1}^*$. The increment $\Xi_t^*=W_{t+1}^*-W_{t}^*$ measures this change.}
  \label{fig:opt2}
\end{subfigure}
\caption{Schematic diagrams of augmented optimization in the parameter space $\R^{n\times p}$.}
\end{figure*}

Given an augmentation scheme for the overparameterized linear model \eqref{E:linear-model}, the time $t$ gradient update at learning rate $\eta_t$ is 
\begin{equation} \label{eq:aug-update}
  W_{t + 1} := W_t + \frac{2\eta_t}{N} \cdot (Y_t - W_t X_t) X_t^\cT,
\end{equation}
where $\mathcal D_t = (X_t,Y_t)$ is the augmented dataset at time $t.$ The
minimum norm minimizer of the corresponding proxy loss $\overline{\cL}_t$
(see \eqref{eq:proxy-opt}) is 
\begin{equation} \label{eq:min-norm-opt}
W^*_t := \E[Y_tX_t^\cT] \E[X_t X_t^\cT]^+,
\end{equation}
where $\E[X_t X_t^\cT]^+$ denotes the Moore-Penrose pseudo-inverse (see Figure \ref{fig:opt}). In this section we state a rigorous result, Theorem \ref{thm:mr-main}, giving
sufficient conditions on the learning rate $\eta_t$ and distributions of the
augmented matrices $X_t, Y_t$ under which augmented gradient descent converges. To state it, note that \eqref{eq:aug-update} implies that each row of
$W_{t + 1} - W_t$ is contained in the column span of the Hessian $X_tX_t^{\cT}$
of the augmented loss and therefore almost surely belongs to the subspace
\begin{equation}\label{E:V-def}
V_{\parallel} := \text{column span of }\E[X_t X_t^\cT] \subseteq \RR^n,
\end{equation}
as illustrated in Figure \ref{fig:opt}. The reason is that, in the orthogonal
complement to $V_{\parallel}$, the augmented loss $\cL(W;\mathcal D_t)$ has zero
gradient with probability $1$. To ease notation, we assume that $V_{\parallel}$ is
independent of $t$. This assumption holds for additive Gaussian noise, random
projection, MixUp, SGD, and their combinations. It is not necessary in general,
however, and we refer the interested reader to Remark \ref{rem:v-par} in the Appendix
for how to treat the general case.

Let us denote by $Q_{\parallel}: \RR^n \to \RR^n$ the orthogonal projection onto
$V_{\parallel}$ (see Figure \ref{fig:opt}). As we already pointed out, at time $t$, gradient descent leaves
the matrix of projections $W_{t}(\Id-Q_{\parallel})$ of each row of $W_t$ onto the orthogonal complement of
$V_{\parallel}$ unchanged. In contrast, $\|W_tQ_{\parallel}-W_t^*\|_F$ decreases 
at a rate governed by the smallest positive eigenvalue $\lambda_{\text{min},V_{\parallel}}\lrr{\Ee{X_tX_t^\cT}}$
of the Hessian for the proxy loss $\overline{\mathcal L}_t$, which is obtained
by restricting its full Hessian $\Ee{X_tX_t^\cT}$ to $V_{\parallel}$. Moreover,
whether and at what rate $W_tQ_{\parallel}-W_t^*$ converges to $0$ depends on how quickly
the distance 
\begin{equation} \label{eq:xi-def}
\Xi_t^* := W_{t + 1}^* - W_t^*
\end{equation}
between proxy loss optima at successive steps tends to zero (see Figure \ref{fig:opt2}). 

\begin{theorem}[Special case of Theorem \ref{thm:mr}] \label{thm:mr-main}
Suppose that $V_\parallel$ is independent of $t$, that the learning rate satisfies
$\eta_t \to 0$, that the proxy optima satisfy
  \begin{equation} \label{eq:mr2-main}
  \sum_{t = 0}^\infty \|\Xi_t^*\|_F < \infty,
\end{equation}
ensuring the existence of a limit $W_\infty^* := \lim_{t \to \infty} W_t^*$, that
  \begin{equation} \label{eq:mr1-main}
    \sum_{t = 0}^\infty \eta_t \lambda_{\text{min}, V_\parallel}(\E[X_t X_t^\cT]) = \infty
  \end{equation}
and finally that
\begin{equation} \label{eq:mr3-main}
\sum_{t = 0}^\infty \eta_t^2 \E\Big[ \|X_t X_t^\cT - \E[X_t X_t^\cT]\|_F^2
 + \|Y_t X_t^\cT - \E[Y_t X_t^\cT]\|_F^2\Big] < \infty.
\end{equation}
Then, for any initialization $W_0$, we have that $W_t Q_\parallel$ converges in probability to $W_\infty^*$.  
\end{theorem}

If the same augmentation is applied
with different strength parameters at each step $t$ (e.g. the noise level $\sigma_t^2$ for additive Gaussian noise), we may specialize Theorem \ref{thm:mr-main}
to this augmentation scheme.  More precisely, translating conditions
(\ref{eq:mr2-main}), (\ref{eq:mr1-main}), (\ref{eq:mr3-main}) 
into conditions on the learning rate and augmentation strength gives
conditions on the schedule for $\eta_t$ and these strength parameters to ensure
convergence.  In \S \ref{sec:add-noise} and \S \ref{sec:projections},
we do this for additive Gaussian noise and random projections.


When the augmentation scheme is static in $t$, Theorem \ref{thm:mr-main}
reduces to a standard Monro-Robbins theorem \cite{robbins1951stochastic} for
the (static) proxy loss $\overline{\cL}_t(W)$. As in that setting, condition
(\ref{eq:mr1-main}) enforces that the learning trajectory travels far enough
to reach an optimum, and the summand in Condition (\ref{eq:mr3-main}) 
bounds the variance of the gradient of the augmented loss $\cL(W; \cD_t)$
to ensure the total variance of the stochastic
gradients is summable. Condition (\ref{eq:mr2-main}) is new and enforces that
the minimizers $W_t^*$ of the proxy losses $\overline{\cL}_t(W)$ change slowly
enough for augmented optimization procedure to keep pace.

\subsection{Convergence rates and scheduling for data augmentation} \label{sec:rates-sched}

Refining the proof of Theorem \ref{thm:mr-main} gives rates of convergence
for the projections $W_t Q_{\parallel}$ of the weights onto $V_{\parallel}$ to
the limiting optimum $W_\infty^*$. When the quantities in
Theorem \ref{thm:mr-main} have power law decay, we obtain the following result.

\begin{theorem}[Special case of Theorem \ref{thm:mr-quant}] \label{thm:mr-quant-informal}
Suppose $V_\parallel$ is independent of $t$ and the learning rate satisfies $\eta_t \to 0$.
Moreover assume that for some  $0 < \alpha < 1 < \beta_1, \beta_2$ and $\gamma > \alpha$ we have
\begin{equation} \label{eq:rate-inf1}
  \eta_t \lambda_{\text{min}, V_\parallel} (\E[X_t X_t^\cT]) = \Omega(t^{-\alpha}),\qquad \|\Xi_t^*\|_F = O(t^{-\beta_1})
\end{equation}
as well as
\begin{equation} \label{eq:rate-inf2}
  \eta_t^2 \E[\|X_t X_t^\cT - \E[X_t X_t^\cT]\|^2_2] = O(t^{-\gamma})
\end{equation}
and
\begin{equation} \label{eq:rate-inf3}
\eta_t^2 \E\Big[ \|\E[W_t] (X_t X_t^\cT - \E[X_t X_t^\cT])
- (Y_t X_t^\cT - \E[Y_t X_t^\cT]) \|_F^2 \Big] = O(t^{-\beta_2}).
\end{equation}
Then, for any initialization $W_0$, we have for any $\eps > 0$ the following convergence in probability:
\[
t^{\min\{\beta_1 - 1, \frac{\beta_2 - \alpha}{2}\} - \eps} \|W_tQ_{\parallel} - W_\infty^* \|_F \overset{p} \to 0.
\]    
\end{theorem}

It may be surprising that $\E[W_t]$ appears in condition \eqref{eq:rate-inf3}. Note
that $\E[W_t]$ is the gradient descent trajectory for the time-varying sequence
of deterministic proxy losses $\overline{\cL}_t(W)$. To apply Theorem \ref{thm:mr-quant-informal},
one may first study this deterministic problem to show that $\E[W_t]$ converges
to $W_\infty^*$ at some rate and then use \eqref{eq:rate-inf3} to obtain rates
of convergence of the true stochastic trajectory $W_t$ to $W_\infty^*$.

In \S \ref{sec:add-noise} and \S \ref{sec:projections} below,
we specialize Theorems \ref{thm:mr-main} and \ref{thm:mr-quant-informal} to obtain
rates of convergence for specific augmentations. Optimizing the learning rate
and augmentation parameter schedules in Theorem \ref{thm:mr-quant-informal} yields
power law schedules with convergence rate guarantees in these settings.

\section{Special Case: Additive Gaussian Noise} \label{sec:add-noise}

We now specialize our main results Theorem \ref{thm:mr-main} and \ref{thm:mr-quant-informal}
to the commonly used additive noise augmentation \citep{grandvalet1997noise}.  Under gradient descent, this
corresponds to taking augmented data matrices
\[
X_t = X + \sigma_t G_t \qquad \text{and} \qquad Y_t = Y,
\]
where $G_t\in \R^{n\times N}$ are independent matrices with i.i.d. standard Gaussian entries,
and $\sigma_t > 0$ is a strength parameter.  Under SGD (with replacement), this corresponds to
augmented data matrices
\[
X_t = c_t(XA_t + \sigma_t \cdot G_t) \qquad \text{and} \qquad Y_t = c_t YA_t,
\]
where $A_t \in \R^{N \times B_t}$ has i.i.d. columns with a single non-zero entry equal to $1$
chosen uniformly at random and $c_t = \sqrt{N/B_t}$ is a normalizing factor.  In both cases,
the proxy loss is
\begin{equation} \label{eq:noise-proxy}
\cL_{\sigma_t}(W) := \frac{1}{N} \|Y - W X\|_F^2 + \sigma_t^2 \|W\|_F^2,
\end{equation}
which to our knowledge was first discovered in \cite{bishop1995training}.

Before stating our precise results, we first illustrate how jointly scheduling learning rate
and augmentation strength impacts GD for overparameterized linear regression, where
\begin{equation}\label{E:overparam}
  N ~=~\#\text{data points} ~<~\text{input dimension }~=~n.      
\end{equation}
The inequality \eqref{E:overparam} ensures $\cL(W;\cD)$ has infinitely many minima, of which
we consider the \emph{minimum norm} minimizer
\[
W_{\min} := YX^\cT (X X^\cT)^+
\]
most desirable. Notice that steps of vanilla gradient descent
\begin{equation} \label{eq:unaug-update}
  W_{t + 1} = W_t - \frac{2\eta_t}{N} \cdot (Y - W_t X) X^\cT
\end{equation}
change the rows of the weight matrix $W_t$ only in the column space $V_{\parallel} = \colspan(XX^\cT) \subseteq \RR^n$.
Because $V_{\parallel}\neq \RR^n$ by the overparameterization assumption \eqref{E:overparam},
minimizing $\cL(W;\cD)$ \textit{without augmentation} cannot change $W_{t, \perp}$, the matrix whose rows are 
the components of the rows of $W_t$ orthogonal to $V_\parallel$. This means that GD converges to
the minimal norm optimizer $W_{\min}$ only when each row of $W_0$ belongs to $V_\parallel$.  As this explicit
initialization may not be available for more general models, we seek to find augmentation
schedules which allow GD or SGD to converge to $W_{\min}$ without it, in the spirit
of recent studies on implicit bias of GD.

\begin{figure}[ht!]
  \centering
  \setlength\tabcolsep{10pt} 
  \begin{tabular}{cc}
    \includegraphics[width=2.5in]{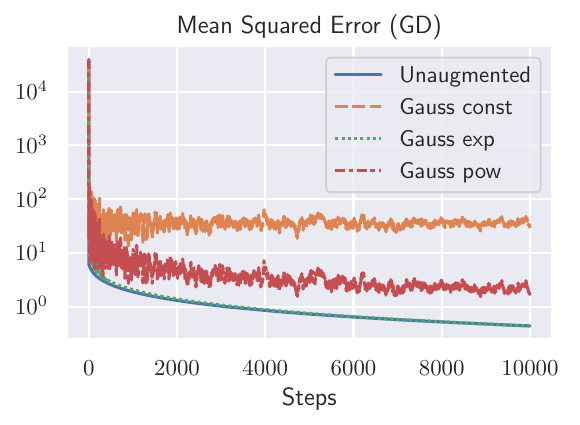} & \includegraphics[width=2.5in]{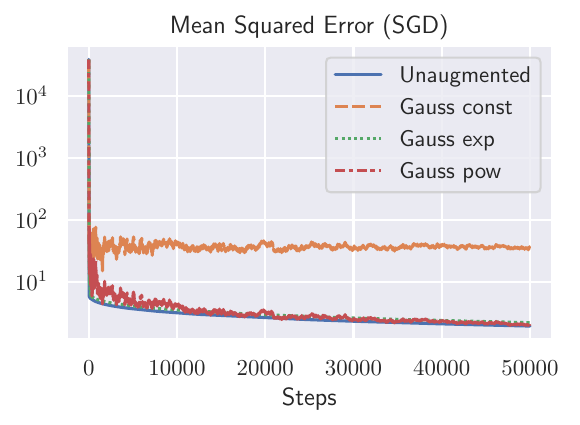} \\
    \includegraphics[width=2.5in]{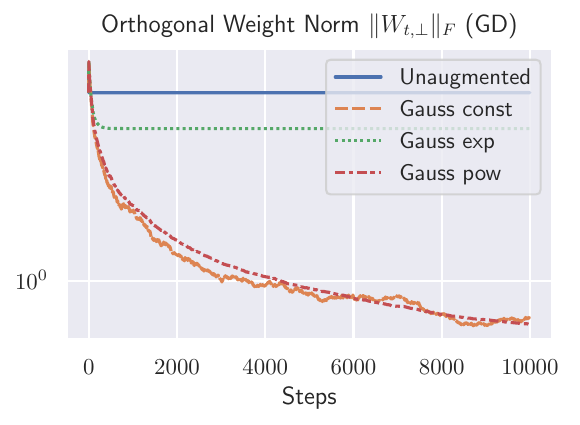} & \includegraphics[width=2.5in]{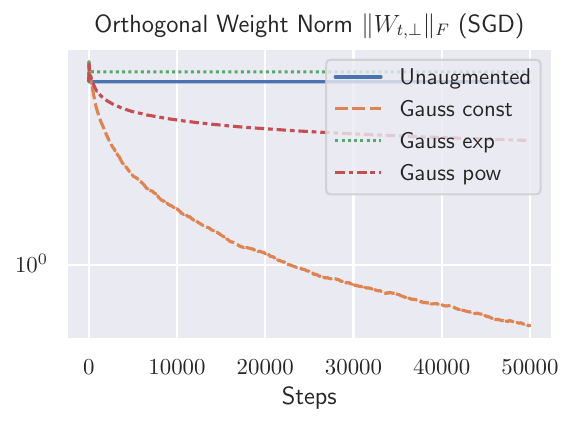}
  \end{tabular}
  \caption{MSE and $\|W_{t, \perp}\|_F$ for optimization trajectories of
    GD with additive Gaussian noise augmentation and SGD with
    additive Gaussian noise augmentation under different augmentation schedules. For both GD and SGD, jointly scheduling learning
    rate and noise variance to have polynomial decay is necessary
    for optimization to converge to the minimal norm solution $W_{\min}$.
    Gauss const, Gauss exp, and Gauss pow have Gaussian noise
    augmentation with $\sigma_t^2 = 2, 2e^{-0.02 t}, 2(1 + \frac{t}{50})^{- 0.33}$, respectively.
    All other details are given in \S \ref{sec:experiment}.
  }
  \label{fig:figure}
\end{figure}

\subsection{Joint schedules for augmented GD with additive noise to converge to $W_{\min}$}

We specialize Theorems \ref{thm:mr-main} and \ref{thm:mr-quant-informal} to additive Gaussian noise to
show that when the learning rate $\eta_t$ and noise strength $\sigma_t$ are jointly scheduled
to converge to $0$ at appropriate rates, augmented gradient descent can find the minimum norm optimizer $W_{\min}$.

\begin{theorem} \label{thm:gauss-gd} 
  Consider any joint schedule of the learning rate $\eta_t$ and noise variance
  $\sigma_t^2$ in which both $\eta_t$ and $\sigma_t^2$ tend to $0$ and $\sigma_t$ is
  non-decreasing. If the joint schedule satisfies
  \begin{equation} \label{eq:gauss-mr}
    \sum_{t = 0}^\infty \eta_t \sigma_t^2 = \infty \qquad \text{and} \qquad
    \sum_{t = 0}^\infty \eta_t^2 \sigma_t^2 < \infty,
  \end{equation}
  then the weights $W_t$ converge in probability
  to the minimal norm optimum $W_{\text{min}}$ regardless of the initialization.
  Moreover, the first condition in \eqref{eq:gauss-mr} is necessary for $\E[W_t]$ to converge to $W_{\min}$.

  If we further have $\eta_t = \Theta(t^{-x})$ and $\sigma_t^2 = \Theta(t^{-y})$ with $x, y > 0$, $x + y < 1$,
  and $2x + y > 1$ so that $\eta_t$ and $\sigma_t^2$ satisfy (\ref{eq:gauss-mr}), then for
  small $\eps > 0$, we have $t^{\min\{y, \frac{1}{2}x\} - \eps} \|W_t - W_{\text{min}}\|_F \overset{p} \to 0$.
\end{theorem}

The conditions of (\ref{eq:gauss-mr}) require that $\eta_t$ and $\sigma_t$ be
jointly scheduled correctly to ensure convergence to $W_{\min}$ and are akin to
the Monro-Robbins conditions \citep{robbins1951stochastic} for convergence
of stochastic gradient methods. 
We now give an heuristic
explanation for why the first condition from \eqref{eq:gauss-mr} is necessary.
In this setting, the \textit{average} trajectory of augmented gradient descent
\begin{equation}\label{E:gauss-avg-gd}
  \E[W_{t+1}] = \E[W_t] -\eta \nabla_{W}\cL_{\sigma_t}(W)\big|_{W = \E[W_t]}
\end{equation}
is given by gradient descent on the ridge regularized losses $\mathcal L_{\sigma_t}(W)$.
If $\sigma_t\equiv \sigma >0$ is constant, then $\E[W_t]$ will
converge to the unique minimizer $W_{\sigma}^*$ of the ridge regularized loss
$\cL_\sigma$. This point $W_\sigma^*$ has zero orthogonal component, but \emph{does not}
minimize the original loss $\cL$.

To instead minimize $\cL$, we must choose a schedule satisfying $\sigma_t\gives 0$.
For the expected optimization trajectory to converge to $W_{\min}$ for such a schedule,
the matrix $\E[W_{t, \perp}]$ of components of rows of $\Ee{W_{t}}$ orthogonal to $V_{\parallel}$ must converge to $0$. The
GD steps for this matrix yield
\begin{equation} \label{E:gauss-perp}
\E[W_{t + 1, \perp}] 
= (1 - \eta_t \sigma_t^2) \E[W_{t, \perp}]
= \prod_{s = 0}^t (1 - \eta_s \sigma_s^2) \E[W_{0, \perp}].
\end{equation}
Because $\eta_t \sigma_t^2$ approaches $0$, this implies the necessary condition
$\sum_{t=0}^\infty \eta_t\sigma_t^2 = \infty$ for $\E[W_{t,\perp}]\gives 0$.

This argument illustrates a key intuition behind the conditions
(\ref{eq:gauss-mr}).  The augmentation strength $\sigma_t$ must decay to $0$ to
allow convergence to a true minimizer of the training loss, but this convergence
must be carefully tuned to allow the implicit regularization of the noising
augmentation to kill the orthogonal component of $W_t$ in expectation. In a similar
manner, the second expression in \eqref{eq:gauss-mr} measures the total variance
of the gradients and ensures that only a finite amount of noise is injected
into the optimization.

Although Theorem \ref{thm:gauss-gd} is stated for additive Gaussian noise,
an analogous version holds for arbitrary additive noise with bounded moments.
Moreover, optimizing over $x, y$, the fastest rate of convergence guaranteed by
Theorem \ref{thm:gauss-gd} is obtained by setting $\eta_t = t^{-2/3 + \eps},\, \sigma_t^2 = t^{-1/3}$
and results in a $O(t^{-1/3+\eps})$ rate of convergence. It is not evident that this
is best possible, however.

\subsection{Joint schedules for augmented SGD with additive noise to converge to $W_{\min}$}
To conclude our discussion of additive noise augmentations, we present the following result
on convergence to $W_{\min}$ in the presence of both Gaussian noise and SGD (where datapoints
in each batch are selected with replacement).
\begin{theorem} \label{thm:gauss-sgd}
Suppose $\sigma_t^2\gives 0$ is decreasing, $\eta_t \to 0$, and we have
\begin{equation} \label{eq:sgd-noise-mr}
  \sum_{t = 0}^\infty \eta_t \sigma_t^2  = \infty \qquad \text{and} \qquad
  \sum_{t = 0}^\infty \eta_t^2  < \infty.
\end{equation}
Then the trajectory $W_t$ of SGD with additive noise converges in probability to $W_{\text{min}}$ for any initialization.  If
we further have $\eta_t = \Theta(t^{-x})$ and $\sigma_t^2 = \Theta(t^{-y})$ with $x > \frac{1}{2}$,
$y > 0$ and $x + y < 1$, then for any $\eps > 0$ we have that $t^{\min\{y, \frac{1}{2}x\} - \eps} \|W_t - W_{\text{min}}\|_F \overset{p} \to 0$.
\end{theorem}

Theorem \ref{thm:gauss-sgd} is the analog of Theorem \ref{thm:gauss-gd} for mini-batch SGD
and provides an example where our framework can handle the \emph{composition} of two augmentations,
namely additive noise and mini-batching.  The difference between conditions (\ref{eq:sgd-noise-mr})
for SGD and (\ref{eq:gauss-mr}) for GD accounts for the fact that the batch selection changes the
scale of the gradient variance at each step.
Finally, Theorem \ref{thm:gauss-sgd} reveals a qualitative difference between SGD with
and without additive noise. If $\eta_t$ has power law decay, the convergence of noiseless SGD
(Theorem \ref{thm:sgd}) is exponential in $t$, while Theorem \ref{thm:gauss-sgd} gives power law rates.

\subsection{Experimental validation} \label{sec:experiment}

To validate Theorems \ref{thm:gauss-gd} and \ref{thm:gauss-sgd}, we ran augmented GD and SGD
with additive Gaussian noise on $N = 100$ simulated datapoints.  Inputs were i.i.d. Gaussian
vectors in dimension $n = 400$, and outputs in dim $p = 1$ were generated by a random linear
map with i.i.d Gaussian coefficients drawn from $\cN(1, 1)$.  The learning rate followed a
 fixed polynomially decaying schedule
$\eta_t = \frac{0.005}{100} \cdot (\text{batch size}) \cdot (1 + \frac{t}{20})^{- 0.66}$, and the batch size
used for SGD was $20$.
Figure \ref{fig:figure} shows MSE and $\|W_{t, \perp}\|_F$ along a single optimization trajectory with
different schedules for the variance $\sigma_t^2$ used in Gaussian noise augmentation.
Complete code to generate this figure is provided in \url{supplement.zip} in the supplement. It
ran in 30 minutes on a standard laptop CPU.

For both GD and SGD, Figure \ref{fig:figure} shows that the optimization trajectory reaches $W_{\min}$ only
when both learning rate and noise variance decay polynomially to zero.  Indeed, 
Figure \ref{fig:figure} shows that if $\sigma_t^2$ is zero (blue) or exponentially decaying (green), then
while the MSE tends to zero, the orthogonal component $W_{t, \perp}$ does not tend to zero.  Thus these
choices of augmentation schedule cause $W_t$ to converge to an optimum which does not have minimal norm.

On the other hand, if $\sigma_t^2$ remains constant (orange), then while $W_{t,\perp}$ tends to zero, the
MSE is not minimized.  Only by decaying both noise strength and learning rate to $0$ at sufficiently slow
polynomial rates (red) prescribed by Theorem \ref{thm:gauss-gd} do we find both MSE and $W_{t, \perp}$
tending to $0$, meaning that augmented (S)GD finds the minimum norm optimum $W_{\min}$ under this choice
of parameter scheduling.

\section{Special Case: Augmentation with Random Projections} \label{sec:projections}

We further illustrate our results by specializing them to a class of augmentations
which replace each input $x$ in a batch by its orthogonal projection $\Pi_tX$ onto
a random subspace. In practice (e.g. when using CutOut \citep{devries2017improved}
or SpecAugment \cite{park2019specaugment}), the subspace is chosen based on a prior
about correlations between components of $X$, but we consider the
simplified case of a uniformly random subspace of $\R^n$ of given dimension.

At each time step $t$ we fix a dimension $k_t$ and a fixed $k_t$-dimensional
subspace $\twiddle{S}_t$ of $\R^n$.  Define the random subspace $S_t$ by
\[
S_t:=Q_t(\twiddle{S}_t)=\set{Q_tx~|~x\in \twiddle{S}_t},
\]
where $Q_t\in O(n)$ is a Haar random orthogonal matrix. Thus, $S_t$ is uniformly
distributed among all $k_t$-dimensional subspaces in $\R^n$. At step $t$, we take
the augmentation given by
\[
X_t = \Pi_t X \qquad Y_t = Y,\qquad \Pi_t := Q_t \wPi_t Q_t^\cT,
\]
where $\wPi_t$ is the orthogonal projection onto $\twiddle{S}_t$ and hence
$\Pi_t$ is the orthogonal projection onto $S_t$.

Denoting by
$\gamma_t = k_t/n$ the relative dimension of $S_t$, a direct computation
(see Lemma \ref{lem:data-moments}) reveals that the proxy loss $\overline{\cL}_t(W)$ equals $\cL(\gamma_tW;\cD)$ plus
\begin{multline} \label{eq:proj-proxy}
\frac{1}{N} \frac{\gamma_t(1 - \gamma_t)}{n} \|X\|_F^2 \cdot \|W\|_F^2+
 \frac{\gamma_t(1 - \gamma_t) (1/n - 2/n^2)}{N(1 + 1/n - 2/n^2)} (\|WX\|_F^2 + \frac{1}{n} \|X\|_F^2 \cdot \|W\|_F^2).
\end{multline}
Neglecting terms of order $O(n^{-1})$, this proxy loss applies a Stein-type shrinkage
on input data by $\gamma_t$ and adds a data-dependent $\ell_2$ penalty. 
For $\gamma_t < 1$, the minimizer of the proxy loss (\ref{eq:proj-proxy}) is
\[
W_{\gamma_t}^* = Y X^\cT \Big(\frac{\gamma_t + 1/n - 2/n^2}{1 + 1/n - 2/n^2} XX^\cT
+ \frac{1 - \gamma_t}{1 + 1/n - 2/n^2} \frac{\|X\|_F^2}{n} \Id\Big)^{-1}.
\]
Again, although $W_{\gamma_t}^*$ does not minimize the original objective
for any $\gamma_t < 1$, the sequence of these proxy optima converges to
the minimal norm optimum in the weak regularization limit.  Namely, we have
$\lim_{\gamma_t \to 1^-} W_{\gamma_t}^* = W_{\min}$.
Specializing our general result Theorem \ref{thm:mr-main} to this setting, we
obtain explicit conditions under
which joint schedules of the normalized rank of the projection and the learning
rate guarantee convergence to the minimum norm optimizer $W_\text{min}$.

\begin{theorem} \label{thm:proj-main}
Suppose that $\eta_t \to 0, \gamma_t \to 1$ with $\gamma_t$ non-decreasing and
\begin{equation} \label{eq:proj-mr}
\sum_{t = 0}^\infty \eta_t  (1 - \gamma_t) = \infty \,\,\, \text{and} \,\,\,
\sum_{t = 0}^\infty \eta_t^2 (1 - \gamma_t) < \infty.
\end{equation}
Then, $W_t \overset{p} \to W_\text{min}$.  Further, if $\eta_t = \Theta(t^{-x})$ and
$\gamma_t = 1 - \Theta(t^{-y})$ with $x, y > 0$, $x + y < 1$, and $2x + y > 1$, then
for small $\eps > 0$, we have that $t^{\min\{y, \frac{1}{2}x\} - \eps} \|W_t - W_{\min}\|_F \overset{p} \to 0$.
\end{theorem}

Comparing the conditions (\ref{eq:proj-mr}) of Theorem \ref{thm:proj-main} to
the conditions (\ref{eq:gauss-mr}) of Theorem \ref{thm:gauss-gd}, we see that
$1 - \gamma_t$ is a measure of the strength of the random projection
preconditioning.  As in that setting, the fastest rates of convergence guaranteed by
Theorem \ref{thm:proj-main} are obtained by setting $\eta_t = t^{-2/3 + \eps}$
and $\gamma_t = 1 - t^{-1/3}$, yielding a $O(t^{-1/3 + \eps})$ rate of convergence.

\section{Discussion and Limitations} \label{sec:discussion}

We have presented a theoretical framework to rigorously analyze the effect of
data augmentation. As can be seen in our main results, our framework applies to
completely general augmentations and relies only on analyzing the first few moments
of the augmented dataset. This allows us to handle augmentations as diverse as
additive noise and random projections as well as their composition in a uniform manner.
We have analyzed some representative examples in detail in this work, but many other
commonly used augmentations may be handled similarly: label-preserving transformations
(e.g. color jitter, geometric transformations) and Mixup \citep{zhang2017mixup},
among many others. Another line of investigation left to future work is to compare
different methods of combining augmentations such as mixing, alternating, or composing,
which often improve performance in the empirical literature \citep{hendrycks2020augmix}.

Though our results provide a rigorous baseline to compare to more complex settings,
the restriction of the present work to linear models is  a significant constraint.
In future work, we hope to extend our general analysis to models closer to those used
in practice.  Most importantly, we intend to consider more complex models such as
kernels (including the neural tangent kernel) and neural networks by making similar
connections to stochastic optimization.  In an orthogonal direction, our analysis
currently focuses on the mean square loss for regression, and we aim to extend it
to other losses such as cross-entropy.  Finally, our study has thus far been
restricted to the effect of data augmentation on optimization, and it would be
of interest to derive consequences for generalization with more complex models.
We hope our framework can provide the theoretical underpinnings for a more
principled understanding of the effect and practice of data augmentation.

\section*{Broader Impact}

Our work provides a new theoretical approach to data augmentation for neural networks.
By giving a better understanding of how this common practice affects optimization, we
hope that it can lead to more robust and interpretable uses of data augmentation in practice.
Because our work is theoretical and generic, we do not envision negative impacts aside
from those arising from improving learning algorithms in general.

\begin{ack}
  It is a pleasure to thank Daniel Park, Ethan Dyer, Edgar Dobriban, and Pokey Rule for a number of insightful conversations about data augmentation. B.H.~was partially supported by NSF grants DMS-1855684 and DMS-2133806 and ONR MURI ``Theoretical Foundations of Deep Learning.'' Y.~S.~was partially supported by NSF grants DMS-1701654/2039183 and DMS-2054838.
\end{ack}

\newpage

\appendix

\section{Analytic lemmas} \label{sec:analytic}

In this section, we present several basic lemmas concerning convergence for certain matrix-valued
recursions that will be needed to establish our main results. For clarity, we first collect some
matrix notations used in this section and throughout the paper.

\subsection{Matrix notations} \label{sec:notations}

Let $M \in \RR^{m \times n}$ be a matrix.  We denote its Frobenius norm by $\|M\|_F$ and its spectral
norm by $\|M\|_2$.  If $m = n$ so that $M$ is square, we denote by $\diag(M)$ the diagonal matrix
with $\diag(M)_{ii} = M_{ii}$. For matrices $A, B, C$ of the appropriate shapes, define 
\begin{equation} \label{eq:matrix-circ}
A \circ (B \otimes C) := B A C
\end{equation}
and 
\begin{equation} \label{eq:matrix-var}
\Var(A) := \E[A^\cT \otimes A] - \E[A^\cT] \otimes \E[A].
\end{equation}
Notice in particular that
\[
\Tr[\Id \circ \Var(A)] = \E[\|A - \E[A]\|^2_F].
\]

\subsection{One- and two-sided decay}

\begin{definition}\label{def:decay}
Let $A_t \in \RR^{n \times n}$ be a sequence of independent random non-negative definite matrices with
\[
\sup_t \norm{A_t}\leq 2\quad \text{almost surely},
\]
let $B_t \in \RR^{p \times n}$ be a sequence of arbitrary matrices, and let $C_t \in \RR^{n \times n}$ be
a sequence of non-negative definite matrices. We say that the sequence of matrices $X_t \in \RR^{p \times n}$
has one-sided decay of type $(\{A_t\}, \{B_t\})$ if it satisfies
\begin{equation} \label{eq:one-sided}
X_{t + 1} = X_t (\Id - \E[A_t]) + B_t.
\end{equation}
We say that a sequence of non-negative definite matrices $Z_t \in \RR^{n \times n}$ has two-sided decay of type $(\{A_t\}, \{C_t\})$ if it satisfies
\begin{equation} \label{eq:two-sided}
Z_{t + 1} = \E[(\Id -A_t) Z_t (\Id - A_t)] + C_t.
\end{equation}
\end{definition}
Intuitively, if a sequence of matrices $X_t$ (resp. $Z_t$) satisfies one decay of type $(\set{A_t}, \set{B_t})$ (resp. two-sided decay of type $(\set{A_t}, \set{C_t})$), then in those directions $u\in \R^n$ for which $\norm{A_tu}$ does not decay too quickly in $t$ we expect that $X_t$ (resp. $Z_t$) will converge to $0$ provided $B_t$ (resp. $C_t$) are not too large. More formally, let us define
\[
V_\parallel := \bigcap_{t = 0}^\infty \ker\left[\prod_{s = t}^\infty (\Id - \E[A_s])\right]
=\left\{u\in \R^n~\bigg|~\lim_{T\gives \infty}\prod_{s=t}^T (\Id-\E[A_s])u = 0,\quad \forall t\geq 1\right\},
\]
and let $Q_\parallel$ be the orthogonal projection onto $V_\parallel$. It is on the space
$V_\parallel$ that that we expect $X_t,Z_t$ to tend to zero if they satisfy one or two-side decay,
and the precise results follows.  

\subsection{Lemmas on Convergence for Matrices with One and Two-Sided Decay}
We state here several results that underpin the proofs of our main results. We begin by giving in Lemmas \ref{lem:rec-matrix-bound1} and \ref{lem:rec-matrix-bound2} two slight variations of the same simple argument that matrices with one or two-sided decay converge to zero. 
\begin{lemma} \label{lem:rec-matrix-bound1}
If a sequence $\{X_t\}$ has one-sided decay of type $(\{A_t\}, \{B_t\})$ with
  \begin{equation} \label{eq:lam-lim2b}
    \sum_{t = 0}^\infty \|B_t\|_F < \infty,
  \end{equation}
  then $\lim_{t \to \infty} X_t Q_\parallel = 0$.     
\end{lemma}
\begin{proof}
  For any $\eps > 0$, choose $T_1$ so that $\sum_{t = T_1}^\infty \|B_t\|_F < \frac{\eps}{2}$ and $T_2$ so
  that for $t > T_2$ we have
  \[
  \left\|\Big(\prod_{s = T_1}^{t} (\Id -\E[A_s])\Big) Q_\parallel\right\|_2 < \frac{\eps}{2} \frac{1}{\|X_0\|_F + \sum_{s = 0}^{T_1 - 1} \|B_s\|_F}.
  \]
  By (\ref{eq:one-sided}), we find that
  \[
  X_{t + 1} = X_0 \prod_{s = 0}^t (\Id -\E[A_s]) + \sum_{s = 0}^t B_s \prod_{r = s + 1}^t (\Id -\E[A_r]),
  \]
  which implies for $t > T_2$ that
  \begin{equation} \label{eq:one-side-bound}
    \|X_{t + 1} Q_\parallel\|_F \leq \|X_0\|_F \left\|\Big(\prod_{s = 0}^t (\Id -\E[A_s])\Big) Q_\parallel\right\|_2
    + \sum_{s = 0}^t \|B_s\|_F \left\|\Big(\prod_{r = s + 1}^t (\Id -\E[A_r])\Big) Q_\parallel \right\|_2.
   \end{equation}
Our assumption that $\norm{A_t}\leq 2$ almost surely implies that for any $T\leq t$
\[\left\|\Big(\prod_{s = 0}^t (\Id -\E[A_s])\Big) Q_\parallel\right\|_2\leq \left\|\Big(\prod_{s = 0}^{T} (\Id -\E[A_s])\Big) Q_\parallel\right\|_2\]
since each term in the product is non-negative-definite. Thus, we find
   \[
   \|X_{t + 1} Q_\parallel\|_F \leq \left[\|X_0\|_F + \sum_{s = 0}^{T_1 - 1} \|B_s\|_F\right]
   \left\|\Big(\prod_{s = T_1}^{t} (\Id -\E[A_s])\Big) Q_\parallel \right\|_2
    + \sum_{s = T_1}^t \|B_s\|_F < \eps.
  \]
  Taking $t \gives \infty$ and then $\eps\gives 0$ implies that $\lim_{t \to \infty} X_tQ_\parallel = 0$, as desired. 
\end{proof}

\begin{lemma} \label{lem:rec-matrix-bound2}
  If a sequence $\{Z_t\}$ has two-sided decay of type $(\{A_t\}, \{C_t\})$ with
  \begin{equation} \label{eq:sq-norm}
   \lim_{T\gives \infty}\E\left[\left\|\Big(\prod_{s = t}^T (\Id - A_s)\Big) Q_\parallel\right\|_2^2\right] = 0 \quad \text{ for all $t \geq 0$}
  \end{equation}
  and
  \begin{equation} \label{eq:lam-lim2}
    \sum_{t = 0}^\infty \Tr(C_t) < \infty,
  \end{equation}
  then $\lim_{t \to \infty} Q_\parallel^\cT Z_t Q_\parallel = 0$.     
\end{lemma}
\begin{proof}
  The proof is essentially identical to that of Lemma \ref{lem:rec-matrix-bound1}. That is, for $\eps > 0$,
  choose $T_1$ so that $\sum_{t = T_1}^\infty \Tr(C_t) < \frac{\eps}{2}$ and choose $T_2$ by
  (\ref{eq:sq-norm}) so that for $t > T_2$ we have
  \[
  \E\left[\left\|\Big(\prod_{s = T_1}^t (\Id - A_s)\Big) Q_\parallel \right\|_2^2\right]
  < \frac{\eps}{2} \frac{1}{\Tr(Z_0) + \sum_{s = 0}^{T_1 - 1} \Tr(C_s)}.
  \]
  Conjugating (\ref{eq:two-sided}) by $Q_\parallel$, we have that
  \begin{multline*}
  Q_\parallel^\cT Z_{t + 1} Q_\parallel = \E\left[Q_\parallel^\cT \Big(\prod_{s = 0}^t (\Id - A_s)\Big)^\cT Z_0 \Big(\prod_{s = 0}^t (\Id - A_s)\Big) Q_\parallel\right]\\
  + \sum_{s = 0}^t \E\left[Q_\parallel^\cT \Big(\prod_{r = s + 1}^{t} (\Id - A_r)\Big)^\cT C_s \Big(\prod_{r = s + 1}^t (\Id - A_r)\Big) Q_\parallel\right].
  \end{multline*}
Our assumption that $\norm{A_t}\leq 2$ almost surely implies that for any $T\leq t$
\[\left\|\Big(\prod_{s = 0}^t (\Id -A_s)\Big) Q\right\|_2\leq \left\|\Big(\prod_{s = 0}^{T} (\Id -A_s)\Big) Q\right\|_2.\]
For $t > T_2$, this implies by taking trace of both sides that
  \begin{align} \label{eq:two-side-bound}
    \Tr( Q_\parallel^\cT Z_{t + 1} Q_\parallel) &\leq \Tr(Z_0) \E\left[\left\|\Big(\prod_{s = 0}^t (\Id - A_s)\Big) Q_\parallel \right\|_2^2\right] 
     + \sum_{s = 0}^t \Tr(C_s) \E\left[\left\|\Big(\prod_{r = s + 1}^t (\Id - A_r)\Big) Q_\parallel\right\|_2^2\right] \\ \nonumber
     &\leq \left[\Tr(Z_0) + \sum_{s = 0}^{T_1 - 1} \Tr(C_s)\right] \E\left[\left\|\Big(\prod_{s = T_1}^t (\Id - A_s)\Big) Q_\parallel \right\|_2^2\right]
     + \sum_{s = T_1}^t \Tr(C_s)\\ \nonumber
    &< \eps,
  \end{align}
  which implies that $\lim_{t \to \infty} Q_\parallel^\cT Z_t Q_\parallel = 0$.
\end{proof}

The preceding Lemmas will be used to provide sufficient conditions for augmented gradient descent to converge as in Theorem \ref{thm:mr} below. Since we are also interested in obtaining rates of convergence, we record here two quantitative refinements of the Lemmas above that will be used in the proof of Theorem \ref{thm:mr-quant}. 

\begin{lemma} \label{lem:rec-matrix-rate}
  Suppose $\{X_t\}$ has one-sided decay of type $(\{A_t\}, \{B_t\})$. Assume also that for some
  $X \geq 0$ and $C > 0$, we have
  \[
  \log \left\|\Big(\prod_{r = s}^t (\Id -\E[A_r])\Big) Q_\parallel\right\|_2 < X - C \int_s^{t + 1} r^{-\alpha} dr
  \]
  and $\|B_t\|_F = O(t^{-\beta})$ for some $0 < \alpha < 1 < \beta$. Then, $\|X_t Q_\parallel\|_F = O(t^{\alpha - \beta})$.  
\end{lemma}
\begin{proof}
Denote $\gamma_{s, t} := \int_s^{t} r^{-\alpha} dr$.  By (\ref{eq:one-side-bound}), we have for some
constants $C_1, C_2 > 0$ that 
\begin{equation}\label{eq:rate-start}
\|X_{t + 1}Q_\parallel\|_F < C_1 e^{-C \gamma_{1, t + 1}} + C_2 e^X \sum_{s = 1}^t (1+s)^{-\beta} e^{- C \gamma_{s + 1, t + 1}}.
\end{equation}
The first term on the right hand side is exponentially decaying in $t$ since $\gamma_{1,t+1}$
grows polynomially in $t$. To bound the second term, observe that the function
\[
f(s):=C\gamma_{s+1,t+1} - \beta\log(s+1)
\]
satisfies
\[
f'(s)\geq 0 \quad \Leftrightarrow\quad C(s+1)^{-\alpha} - \frac{\beta}{1+s}\geq 0\quad \Leftrightarrow\quad s \geq \lrr{\frac{\beta}{C}}^{1/(1-\alpha)}=:K.
\]
Hence, the summands are monotonically increasing for $s$ greater than a fixed constant $K$ depending only on $\alpha,\beta,C$. Note that
\[
\sum_{s = 1}^K (1+s)^{-\beta} e^{- C \gamma_{s + 1, t + 1}}\leq Ke^{- C \gamma_{K + 1, t + 1}}\leq K e^{-C' t^{1-\alpha}}
\]
for some $C'$ depending only on $\alpha$ and $K$, and hence sum is exponentially decaying in $t$. Further, using an integral comparison, we find
\begin{equation}\label{eq:one-sided-est}
\sum_{s = K+1}^t (1+s)^{-\beta} e^{- C \gamma_{s + 1, t + 1}}\leq \int_K^{t} (1+s)^{-\beta}e^{-\frac{C}{1-\alpha}\lrr{(t+1)^{1-\alpha}- (s+1)^{1-\alpha}}}ds.
\end{equation}
Changing variables using $u=(1+s)^{1-\alpha}/(1-\alpha)$, the last integral has the form
\begin{equation}\label{eq:one-sided-est-1}
e^{-Cg_t} (1-\alpha)^{-\xi}\int_{g_K}^{g_t} u^{-\xi}e^{Cu} du,\qquad g_x:=\frac{(1+x)^{1-\alpha}}{1-\alpha},\, \xi:=\frac{\beta-\alpha}{1-\alpha}.
\end{equation}
Integrating by parts, we have
\[\int_{g_K}^{g_t} u^{-\xi}e^u du = C^{-1}\xi\int_{g_K}^{g_t} u^{-\xi-1}e^{Cu} du + (u^{-\xi}e^{Cu})|_{g_K}^{g_t}\]
Further, since on the range $g_K\leq u\leq g_t$ the integrand is increasing, we have
\[
 e^{-Cg_t}\xi\int_{g_K}^{g_t} u^{-\xi-1}e^{Cu}du \leq \xi g_t^{-\xi}.
\]
Hence, $e^{-Cg_t}$ times the integral in \eqref{eq:one-sided-est-1} is bounded above by 
\[ O(g_t^{-\xi})+e^{-Cg_t} (u^{-\xi}e^{Cu})|_{g_K}^{g_t} =O(g_t^{-\xi}) .\]
Using \eqref{eq:one-sided-est} and substituting the previous line into \eqref{eq:one-sided-est-1} yields the estimate
\[\sum_{s = K+1}^t (1+s)^{-\beta} e^{- C \gamma_{s + 1, t + 1}}\leq (1+t)^{-\beta+\alpha},\]
which completes the proof. 
\end{proof}

\begin{lemma} \label{lem:rec-matrix-rate2}
 Suppose $\{Z_t\}$ has two-sided decay of type $(\{A_t\}, \{C_t\})$. Assume also that for some $X\geq 0$ and $C > 0$, we have
  \[
  \log \E\left[\left\|\Big(\prod_{r = s}^t (\Id - A_r)\Big) Q_\parallel\right\|_2^2\right] < X - C \int_s^{t + 1} r^{-\alpha} dr
  \]
  as well as $\Tr(C_t) = O(t^{-\beta})$ for some $0 < \alpha < 1 < \beta$. Then $\tr(Q_\parallel^TZ_tQ_\parallel) = O(t^{\alpha - \beta})$.  
\end{lemma}
\begin{proof}
This argument is identical to the proof of Lemma \ref{lem:rec-matrix-rate}. Indeed, using \eqref{eq:two-side-bound} we have that
\[
\tr\lrr{Q_\parallel^TZ_tQ_\parallel}\leq  C_1 e^{-C \gamma_{1, t + 1}} + C_2e^X \sum_{s = 1}^t (1+s)^{-\beta} e^{- C \gamma_{s + 1, t + 1}}.
\]
The right hand side of this inequality coincides with the expression on the right hand side of
\eqref{eq:rate-start}, which we already bounded by $O(t^{\beta-\alpha})$ in the proof of
Lemma \ref{lem:rec-matrix-rate}. 
\end{proof}

In what follows, we will use a concentration result for products of matrices from \cite{huang2020matrix}.
Let $Y_1, \ldots, Y_n \in \RR^{N \times N}$ be independent random matrices.  Suppose that
\[
\|\E[Y_i] \|_2 \leq a_i \qquad \text{and} \qquad \E\left[ \|Y_i - \E[Y_i]\|_2^2 \right] \leq b_i^2 a_i^2
\]
for some $a_1, \ldots, a_n$ and $b_1, \ldots, b_n$. We will use the following result, which is
a specialization of Theorem 5.1 in \cite{huang2020matrix} for $p = q = 2$.

\begin{theorem}[{Theorem 5.1 in \cite{huang2020matrix}}] \label{thm:mat-prod}
For $Z_0 \in \RR^{N \times n}$, the product $Z_n = Y_n Y_{n - 1} \cdots Y_1 Z_0$ satisfies
\begin{align*}
  \E\left[\|Z_n\|_2^2\right] &\leq e^{\sum_{i = 1}^n b_i^2} \prod_{i = 1}^n a_i^2 \cdot \|Z_0\|_2^2\\
  \E\left[\|Z_n - \E[Z_n]\|_2^2\right] &\leq \Big(e^{\sum_{i = 1}^n b_i^2} - 1\Big) a_i^2 \cdot \|Z_0\|_2^2.
\end{align*}
\end{theorem}

Finally, we collect two simple analytic lemmas for later use.

\begin{lemma} \label{lem:spec-bound}
  For any matrix $M \in \RR^{m \times n}$, we have that
  \[
  \E[\|M\|_2^2] \geq \| \E[M] \|_2^2.
  \]
\end{lemma}
\begin{proof}
  We find by Cauchy-Schwartz and the convexity of the spectral norm that
  \[
  \E[\|M\|_2^2] \geq \E[\|M\|_2]^2  \geq \| \E[M] \|_2^2. \qedhere
  \]
\end{proof}

\begin{lemma} \label{lem:sum-bound}
  For bounded $a_t \geq 0$, if we have $\sum_{t = 0}^\infty a_t = \infty$, then for any $C > 0$ we have
  \[
  \sum_{t = 0}^\infty a_t e^{- C \sum_{s = 0}^t a_s} < \infty.
  \]
\end{lemma}
\begin{proof}
  Define $b_t := \sum_{s = 0}^t a_s$ so that
  \[
  S := \sum_{t = 0}^\infty a_t e^{- C \sum_{s = 0}^t a_s} = \sum_{t = 0}^\infty (b_t - b_{t - 1}) e^{-C b_t} \leq \int_0^{\infty} e^{-C x} dx < \infty,
  \]
  where we use $\int_0^\infty e^{-Cx} dx$ to upper bound its right Riemann sum.
\end{proof}


\section{Analysis of data augmentation as stochastic optimization} \label{sec:rates}

In this section, we prove generalizations of our main theoretical results Theorems
\ref{thm:mr-main} and \ref{thm:mr-quant-informal} giving Monro-Robbins type conditions
for convergence and rates for augmented gradient descent in the linear setting. 

\subsection{Monro-Robbins type results} \label{sec:res-conv}

To state our general Monro-Robbins type convergence results, let us briefly recall the notation. We consider overparameterized linear regression with loss
\[\mathcal L(W;\mathcal D) = \frac{1}{N}\norm{WX-Y}_F^2,\]
where the dataset $\mathcal D$ of size $N$ consists of data matrices $X,Y$ that each have $N$ columns $x_i\in \R^n,y_i\in \R^p$ with $n >N.$ We optimize $\mathcal L(W;\mathcal D)$ by augmented gradient descent, which means that at each time $t$ we replace $\mathcal D = (X,Y)$ by a random dataset $\mathcal D_t = (X_t,Y_t)$. We then take a step 
\[W_{t+1}= W_t -\eta_t \nabla_W \mathcal L(W_t;\mathcal D_t)\]
of gradient descent on the resulting randomly augmented loss $\mathcal L(W;\mathcal D_t)$ with learning rate $\eta_t$. Recall that we set
\[
V_\parallel := \text{ column span of $\E[X_t X_t^\cT]$}
\]
and denoted by $Q_\parallel$ the orthogonal projection onto $V_\parallel$. As noted in \S \ref{sec:mr}, on $V_{\parallel}$ the proxy loss 
\[\overline{\mathcal L}_t = \Ee{\mathcal L(W;\mathcal D_t)}\]
is strictly convex and has a unique minimum, which is
\[W_t^* = \Ee{Y_tX_t^T}(Q_{||} \Ee{X_tX_t^{\cT}} Q_{||})^{-1}.\]
The change from one step of augmented GD to the next in these proxy optima is captured by
\[\Xi_t^*:=W_{t+1}^*-W_t^*.\]
With this notation, we are ready to state Theorems \ref{thm:mr}, which gives two different sets
of time-varying Monro-Robbins type conditions under which the optimization trajectory $W_t$ converges
for large $t$. In Theorem \ref{thm:mr-quant}, we refine the analysis to additionally
give rates of convergence.  Note that Theorem \ref{thm:mr} is a generalization of
Theorem \ref{thm:mr-main} and that Theorem \ref{thm:mr-quant} is a generalization
of Theorem \ref{thm:mr-quant-informal}.

\begin{theorem} \label{thm:mr}
Suppose that $V_\parallel$ is independent of $t$, that the learning rate satisfies $\eta_t \to 0$, that the proxy optima satisfy
  \begin{equation} \label{eq:mr2-app}
  \sum_{t = 0}^\infty \|\Xi_t^*\|_F < \infty,
\end{equation}
ensuring the existence of a limit $W_\infty^* := \lim_{t \to \infty} W_t^*$ and that
  \begin{equation} \label{eq:mr1-app}
    \sum_{t = 0}^\infty \eta_t \lambda_{\text{min}, V_\parallel}(\E[X_t X_t^\cT]) = \infty.
  \end{equation}
Then if either
\begin{equation} \label{eq:mr3-app}
  \sum_{t = 0}^\infty \eta_t^2 \E\left[ \|X_t X_t^\cT - \E[X_t X_t^\cT]\|_F^2 + \|Y_t X_t^\cT - \E[Y_t X_t^\cT]\|_F^2\right] < \infty
\end{equation}
or 
\begin{equation} \label{eq:mr3-opt}
  \sum_{t = 0}^\infty \eta_t^2 \E\Big[ \|X_t X_t^\cT - \E[X_t X_t^\cT]\|_F^2
    + \Big\|\E[W_t] (X_t X_t^\cT - \E[X_t X_t^\cT]) - (Y_t X_t^\cT - \E[Y_t X_t^\cT]) \Big\|_F^2 \Big] < \infty
\end{equation}
hold, then for any initialization $W_0$, we have $W_t Q_\parallel \overset{p} \to W_\infty^*$.  
\end{theorem}



\begin{remark} \label{rem:v-par}
 In the general case, the column span $V_{||}$ of $\E[X_t X_t^\cT]$ may vary with $t$.  This means
 that some directions in $\RR^n$ may only have non-zero overlap with $\colspan(\E[X_t X_t^\cT])$
 for some positive but finite collection of values of $t$. In this case, only finitely many steps of the optimization would move $W_t$ in this direction, meaning that we must define a smaller space for convergence. The correct definition of this subspace turns out to be the following
  \begin{align} 
\label{eq:mr1}    V_\parallel &:= \bigcap_{t = 0}^\infty \ker\left[\prod_{s = t}^\infty \Big(\Id - \frac{2\eta_s}{N} \E[X_s X_s^\cT]\Big)\right]\\
\notag & =\bigcap_{t = 0}^\infty \left\{u\in \RR^n ~\bigg|~ \lim_{T\gives \infty}\prod_{s = t}^T\Big(\Id - \frac{2\eta_s}{N} \E[X_s X_s^\cT]\Big)u = 0\right\}.
  \end{align}
With this re-definition of $V_{||}$ and with $Q_\parallel$ still denoting the orthogonal projection to $V_\parallel$, Theorem \ref{thm:mr} holds verbatim and with the same proof. Note that if $\eta_t \to 0$, $V_{||} = \colspan(\E[X_t X_t^\cT])$ is fixed in $t$, and  (\ref{eq:mr1-app}) holds, this definition of $V_\parallel$ reduces to that defined in (\ref{E:V-def}).
\end{remark}
\begin{remark} 
The condition \eqref{eq:mr3-opt} can be written in a more conceptual way as
\[
\sum_{t = 0}^\infty\left[\|X_t X_t^\cT - \E[X_t X_t^\cT]\|_F^2 + \eta_t^2 \Tr\left[\Id \circ \Var\Big((\E[W_t] X_t  - Y_t)X_t^\cT\Big)\right]\right] < \infty,
\]
where we recognize that $(\E[W_t] X_t  - Y_t)X_t^\cT$ is precisely the stochastic gradient
estimate at time $t$ for the proxy loss $\overline{\mathcal L}_t$, evaluated at $\Ee{W_t}$,
which is the location at time $t$ for vanilla GD on $\overline{\cL}_t$ since taking
expectations in the GD update equation (\ref{eq:aug-update}) coincides with GD for
$\overline{\cL}_t$. Moreover, condition \eqref{eq:mr3-opt} actually implies condition
\eqref{eq:mr3-app} (see \eqref{eq:c-bound} below).
The reason we state Theorem \ref{thm:mr} with both conditions, however, is that \eqref{eq:mr3-opt}
makes explicit reference to the average $\Ee{W_t}$ of the augmented trajectory. Thus, when applying
Theorem \ref{thm:mr} with this weaker condition, one must separately estimate the behavior
of this quantity.
\end{remark}

Theorem \ref{thm:mr} gave conditions on joint learning rate and data augmentation
schedules under which augmented optimization is guaranteed to converge. Our next
result proves rates for this convergence.

\begin{theorem} \label{thm:mr-quant}
Suppose that $\eta_t \to 0$ and that for some $0 < \alpha < 1 < \beta_1, \beta_2$ and $C_1, C_2 > 0$, we have
  \begin{equation} \label{eq:alpha2}
    \log \E\left[\left\|\Big(\prod_{r = s}^t \Big(\Id - \frac{2\eta_r}{N} X_r X_r^\cT\Big)\Big) Q_\parallel\right\|_2^2\right]
    < C_1 - C_2 \int_s^{t + 1} r^{-\alpha} dr
  \end{equation}
  as well as
  \begin{equation} \label{eq:beta1}
    \|\Xi_t^*\|_F = O(t^{-\beta_1})
\end{equation}
and
\begin{equation} \label{eq:beta2}
  \eta_t^2 \Tr\left[\Id \circ \Var(\E[W_t] X_t X_t^\cT - Y_t X_t^\cT\Big)\right] = O(t^{-\beta_2}).
\end{equation}
Then, for any initialization $W_0$, we have for any $\eps > 0$ that 
\[
t^{\min\{\beta_1 - 1, \frac{\beta_2 - \alpha}{2}\} - \eps} \|W_t Q_\parallel - W_\infty^*\|_F \overset{p} \to 0.
\]
\end{theorem}
\begin{remark} \label{rem:reduction}
  To reduce Theorem \ref{thm:mr-quant-informal} to Theorem \ref{thm:mr-quant}, we notice that
  (\ref{eq:rate-inf1}) and (\ref{eq:rate-inf2}) mean that Theorem \ref{thm:mat-prod} applies
  to $Y_t = \Id - 2 \eta_t \frac{X_t X_t^\cT}{N}$ with $a_t = 1 - \Omega(t^{-\alpha})$ and
  and $b_t^2 = O(t^{-\gamma})$, thus implying (\ref{eq:alpha2}).  
\end{remark}

The first step in proving both Theorem \ref{thm:mr} and Theorem \ref{thm:mr-quant} is to obtain
recursions for the mean and variance of the difference $W_t - W_t^*$ between the time $t$ proxy
optimum and the augmented optimization trajectory at time $t.$ We will then complete the proof
of Theorem \ref{thm:mr} in \S \ref{sec:thm-mr-pf} and the proof of Theorem \ref{thm:mr-quant}
in \S \ref{sec:thm-mr-quant-pf}.

\subsection{Recursion relations for parameter moments}

The following proposition shows that difference between the mean augmented dynamics $\E[W_t]$ and the time$-t$ optimum $W_t^*$ satisfies, in the sense of Definition \ref{def:decay}, one-sided decay of type $(\set{A_t}, \set{B_t})$ with
\[
A_t = \frac{2\eta_t}{N} X_t X_t^\cT,\qquad B_t =  - \Xi_t^*.
\]
It also shows that the variance of this difference, which is non-negative definite, satisfies two-sided decay of type $(\set{A_t},\set{C_t})$ with $A_t$ as before and 
\[
C_t = \frac{4\eta_t^2}{N^2} \left[\Id \circ \Var\Big(\E[W_t] X_t X_t^\cT - Y_t X_t^\cT\Big)\right].
\]
In terms of the notations of Appendix \ref{sec:notations}, we have the following recursions.

\begin{prop} \label{prop:mr-moments}
  The quantity $\E[W_t] - W_t^*$ satisfies
  \begin{equation} \label{eq:exp-rec}
  \E[W_{t + 1}] - W_{t + 1}^* = (\E[W_t] - W_t^*) \Big(\Id - \frac{2\eta_t}{N} \E[X_t X_t^\cT]\Big) - \Xi_t^*
  \end{equation}
  and $Z_t := \E[(W_t - \E[W_t])^\cT (W_t - \E[W_t])]$ satisfies
  \begin{equation} \label{eq:sec-rec}
  Z_{t + 1} = \E\left[(\Id - \frac{2\eta_t}{N} X_t X_t^\cT) Z_t (\Id - \frac{2\eta_t}{N} X_t X_t^\cT)\right]
  + \frac{4\eta_t^2}{N^2} \left[\Id \circ \Var\Big(\E[W_t] X_t X_t^\cT - Y_t X_t^\cT\Big)\right].
  \end{equation}
\end{prop}
\begin{proof}
Notice that $\E[X_t X_t^\cT]u = 0$ if and only if $X_t^\cT u = 0$ almost surely, which implies that 
\[
W_t^* \E[X_t X_t^\cT] = \E[Y_t X_t^\cT] \E[X_t X_t^\cT]^+ \E[X_t X_t^\cT] = \E[Y_t X_t^\cT].
\]
Thus, the learning dynamics (\ref{eq:aug-update}) yield
\begin{align*}
  \E[W_{t + 1}] &= \E[W_t] - \frac{2\eta_t}{N} \Big(\E[W_t] \E[X_t X_t^\cT] - \E[Y_t X_t^\cT]\Big)\\
  &= \E[W_t] - \frac{2\eta_t}{N} (\E[W_t] - W_t^*) \E[X_t X_t^\cT].
\end{align*}
Subtracting $W_{t+1}^*$ from both sides yields \eqref{eq:exp-rec}. We now analyze the fluctuations. Writing $\Sym(A) := A + A^\cT$, we have
\begin{multline*}
  \E[W_{t + 1}]^\cT \E[W_{t + 1}] = \E[W_t]^\cT \E[W_t] + \frac{2\eta_t}{N} \Sym\Big(\E[W_t]^\cT \E[Y_t X_t^\cT]
  - \E[W_t]^\cT \E[W_t] \E[X_t X_t^\cT]\Big)\\
  + \frac{4\eta_t^2}{N^2} \Big(\E[X_t X_t^\cT] \E[W_t]^\cT \E[W_t] \E[X_t X_t^\cT] + \E[X_tY_t^\cT]\E[Y_tX_t^\cT]
  - \Sym(\E[X_t X_t^\cT] \E[W_t]^\cT \E[Y_t X_t^\cT])\Big).
\end{multline*}
Similarly, we have that
\begin{multline*}
  \E[W_{t + 1}^\cT W_{t + 1}] = \E[W_t^\cT W_t] + \frac{2\eta_t}{N} \Sym(\E[W_t^\cT Y_t X_t^\cT - W_t^\cT W_t X_t X_t^\cT])\\
  + \frac{4\eta_t^2}{N^2} \E[X_t X_t^\cT W_t^\cT W_t X_t X_t^\cT - \Sym(X_t X_t^\cT W_t^\cT Y_t X_t^\cT)
    + X_t Y_t^\cT Y_t X_t^\cT].
\end{multline*}
Noting that $X_t$ and $Y_t$ are independent of $W_t$ and subtracting yields the desired.
\end{proof}

\subsection{Proof of Theorem \ref{thm:mr}} \label{sec:thm-mr-pf}

First, by Proposition \ref{prop:mr-moments}, we see that $\E[W_t] - W_t^*$ has one-sided decay with
\[
A_t = 2\eta_t \frac{X_t X_t^\cT}{N} \qquad \text{and} \qquad B_t = -\Xi_t^*.
\]
Thus, by Lemma \ref{lem:rec-matrix-bound1} and (\ref{eq:mr2-app}), we find that
\begin{equation} \label{eq:conv-exp}
\lim_{t \to \infty} (\E[W_t] Q_\parallel - W_t^*) = 0,
\end{equation}
which gives convergence in expectation.

For the second moment, by Proposition \ref{prop:mr-moments}, we see that $Z_t$ has two-sided decay with
\[
A_t = 2\eta_t \frac{X_tX_t^\cT}{N} \qquad \text{and}
\qquad C_t = \frac{4\eta_t^2}{N^2} \left[\Id \circ \Var\Big(\E[W_t] X_t X_t^\cT - Y_t X_t^\cT\Big)\right].
\]
We now verify (\ref{eq:sq-norm}) and (\ref{eq:lam-lim2}) in order to apply Lemma \ref{lem:rec-matrix-bound2}.

For (\ref{eq:sq-norm}), for any $\eps > 0$, notice that
\[
\E[\|A_s - \E[A_s]\|_F^2] = \eta_s^2 \E[\| X_s X_s^\cT - \E[X_s X_s^\cT]\|_F^2]
\]
so by either (\ref{eq:mr3-app}) or (\ref{eq:mr3-opt}) we may choose $T_1 > t$ so that
$\sum_{s = T_1}^\infty \E[\|A_s - \E[A_s]\|_F^2] < \frac{\eps}{2}$. Now choose $T_2 > T_1$ so
that for $T > T_2$, we have
\[
\left\| \Big(\prod_{r = T_1}^T \E[\Id - A_r]\Big) Q_\parallel\right\|_2^2
< \frac{\eps}{2} \frac{1}{\|\prod_{s = t}^{T_1 - 1} \E[\Id - A_s]\|_F^2 + \sum_{s = t}^{T_1 - 1} \E[\|A_s - \E[A_s]\|_F^2]}.
\]
For $T > T_2$, we then have
\begin{align*}
  &\E\left[\left\|\Big(\prod_{s = t}^T (\Id - A_s)\Big) Q_\parallel\right\|_2^2\right]\\
  &\leq \left\|\Big(\prod_{s = t}^T \E[\Id - A_s]\Big) Q_\parallel\right\|^2
   + \sum_{s = t}^T \E\left[\left\|\prod_{r = t}^{s} (\Id - A_r) \prod_{r = s + 1}^T (\Id - \E[A_r]) Q_\parallel\right\|_F^2 - \left\|\prod_{r = t}^{s - 1} (\Id - A_r) \prod_{r = s}^T (\Id - \E[A_r]) Q_\parallel\right\|_F^2\right]\\
  &=\left\|\Big(\prod_{s = t}^T \E[\Id - A_s]\Big) Q_\parallel\right\|_F^2 + \sum_{s = t}^T \E\left[\left\|\prod_{r = t}^{s - 1} (\Id - A_r) (A_s - \E[A_s]) \prod_{r = s + 1}^T (\Id - \E[A_r]) Q_\parallel\right\|_F^2\right]\\
  &\leq \left\|\prod_{s = t}^{T_1 - 1} \E[\Id - A_s]\right\|^2 \left\|_F \Big(\prod_{r = T_1}^T \E[\Id - A_r]\Big) Q_\parallel\right\|_2^2 + \sum_{s = t}^T \E[\|A_s - \E[A_s]\|_F^2] \left\| \Big(\prod_{r = s + 1}^T \E[\Id - A_r]\Big) Q_\parallel\right\|_2^2\\
  &\leq \Big(\left\|\prod_{s = t}^{T_1 - 1} \E[\Id - A_s]\right\|_F^2 + \sum_{s = t}^{T_1 - 1} \E[\|A_s - \E[A_s]\|_F^2]\Big) \left\| \Big(\prod_{r = T_1}^T \E[\Id - A_r]\Big) Q_\parallel\right\|_2^2 + \sum_{s = T_1}^T \E[\|A_s - \E[A_s]\|_F^2]\\
  &< \eps,
\end{align*}
which implies (\ref{eq:sq-norm}).  Condition (\ref{eq:lam-lim2}) follows from either (\ref{eq:mr3-opt}) or (\ref{eq:mr3-app}) and the bounds
\begin{align} \label{eq:c-bound}
  \Tr(C_t) &\leq \frac{8\eta_t^2}{N^2} \left( \|\E[W_t] (X_t X_t^\cT - \E[X_t X_t^\cT])\|_F^2 + \|Y_t X_t^\cT - \E[Y_t X_t^\cT]\|_F^2\right)\\  \nonumber
  &\leq \frac{8\eta_t^2}{N^2} \left( \|\E[W_t]\|^2 \|X_t X_t^\cT - \E[X_t X_t^\cT]\|_F^2 + \|Y_t X_t^\cT - \E[Y_t X_t^\cT]\|_F^2\right),
\end{align}
where in the first inequality we use the fact that $\|M_1 - M_2\|_F^2 \leq 2(\|M_1\|_F^2 + \|M_2\|_F^2)$. 
Furthermore, iterating (\ref{eq:exp-rec}) yields $\|\E[W_t] - W_t^*\|_F \leq \|W_0 - W_0^*\|_F + \sum_{t = 0}^\infty \|\Xi_t^*\|_F$, which
combined with (\ref{eq:c-bound}) and either (\ref{eq:mr3-app}) or \eqref{eq:mr3-opt} therefore
implies (\ref{eq:lam-lim2}). We conclude by Lemma \ref{lem:rec-matrix-bound2} that
\begin{equation} \label{eq:conv-var}
\lim_{t \to \infty} Q_\parallel^\cT Z_t Q_\parallel = \lim_{t \to \infty} \E[Q_\parallel^\cT (W_t - \E[W_t])^\cT (W_t - \E[W_t]) Q_\parallel] = 0.
\end{equation}
Together, (\ref{eq:conv-exp}) and (\ref{eq:conv-var}) imply that $W_t Q_\parallel - W_t^* \overset{p} \to 0$.
The conclusion then follows from the fact that $\lim_{t \to 0} W_t^* = W_\infty^*$. This complete the proof of Theorem \ref{thm:mr}. \hfill $\square$

\subsection{Proof of Theorem \ref{thm:mr-quant}}\label{sec:thm-mr-quant-pf}

By Proposition \ref{prop:mr-moments}, $\E[W_t] - W_t^*$ has one-sided decay with 
\[
A_t = \frac{2\eta_t}{N} X_t X_t^\cT,\qquad B_t = - \Xi_t^*.
\]
By Lemma \ref{lem:spec-bound} and (\ref{eq:alpha2}), $\E[A_t]$ satisfies
\begin{align*}
  \log \left\|\prod_{r = s}^t \Big(\Id - 2 \eta_r \frac{1}{N} \E[X_r X_r^\cT]\Big) Q_\parallel \right\|_2 
  &\leq \frac{1}{2} \log \E\left[\left\|\Big(\prod_{r = s}^t \Big(\Id - 2 \eta_r \frac{X_rX_r^\cT}{N}\Big)\Big) Q_\parallel \right\|_2^2\right]\\
  &< \frac{C_1}{2} - \frac{C_2}{2} \int_s^{t + 1} r^{-\alpha} dr.
\end{align*}
Applying Lemma \ref{lem:rec-matrix-rate} using this bound and (\ref{eq:beta1}), we find that
\[
\|\E[W_t] Q_\parallel - W_t^*\|_F = O(t^{\alpha - \beta_1}).
\]
Moreover, because $\|\Xi_t^*\|_F = O(t^{-\beta_1})$, we also find that $\|W_t^* - W_\infty^*\|_F = O(t^{-\beta_1 + 1})$, and hence
\[
\|\E[W_t] Q_\parallel - W_\infty^*\|_F = O(t^{-\beta_1 + 1}).
\]
Further, by Proposition \ref{prop:mr-moments}, $\E[(W_t - \E[W_t])^\cT(W_t - \E[W_t])]$ has two-sided
decay with 
\[
A_t = \frac{2\eta_t}{N} X_t X_t^\cT,
\qquad C_t = \frac{4\eta_t^2}{N^2} \left[\Id \circ \Var\Big(\E[W_t] X_t X_t^\cT - Y_t X_t^\cT\Big)\right].
\]
Applying Lemma \ref{lem:rec-matrix-rate2} with (\ref{eq:alpha2}) and (\ref{eq:beta2}), we find that
\[
\E\left[\|(W_t - \E[W_t]) Q_\parallel\|_F^2\right] = O(t^{\alpha - \beta_2}).
\]
By Chebyshev's inequality, for any $x > 0$ we have
\[
\P\Big(\|W_t Q_\parallel - W_\infty^*\|_F \geq O(t^{-\beta_1 + 1}) + x \cdot O(t^{\frac{\alpha - \beta_2}{2}})\Big) \leq x^{-2}.
\]
For any $\eps > 0$, choosing $x = t^{\delta}$ for small $0 < \delta < \eps$ we find as desired that
\[
t^{\min\{\beta_1 - 1, \frac{\beta_2 - \alpha}{2}\} - \eps} \|W_t Q_\parallel - W_\infty^*\|_F \overset{p} \to 0,
\]
thus completing the proof of Theorem \ref{thm:mr-quant}.\hfill$\square$

\section{Intrinsic time}

Theorem \ref{thm:mr-quant-informal} measures rates in terms of optimization steps $t$,
but a different measurement of time called the \emph{intrinsic time} of the optimization
will be more suitable for measuring the behavior of optimization quantities.  This was
introduced for SGD in \cite{smith2018bayesian, smith2018don}, and we now generalize it
to our broader setting. For gradient descent on a loss $\cL$, the intrinsic
time is a quantity which increments by $\eta \lambda_{\text{min}}(H)$ for a optimization
step with learning rate $\eta$ at a point where $\cL$ has Hessian $H$. When specialized
to our setting, it is given by 
\begin{equation} \label{eq:intrinsic-def}
\tau(t) := \sum_{s = 0}^{t - 1} \frac{2\eta_s}{N} \lambda_{\text{min}, V_\parallel}(\E[X_sX_s^\cT]).
\end{equation}
Notice that intrinsic time of augmented optimization for the sequence of proxy losses
$\overline{\cL}_s$ appears in Theorems \ref{thm:mr-main} and \ref{thm:mr-quant-informal},
which require via conditions \eqref{eq:mr1-main} and \eqref{eq:rate-inf2} that the intrinsic
time tends to infinity as the number of optimization steps grows.

Intrinsic time will be a sensible variable in which to measure the behavior of quantities
such as the fluctuations of the optimization path $f(t) := \E[\|(W_t - \E[W_t])Q_{\parallel}\|_F^2]$.
In the proofs of Theorems \ref{thm:mr-main} and \ref{thm:mr-quant-informal}, we show
that the fluctuations satisfy an inequality of the form 
\begin{equation} \label{eq:int-rec}
f(t+1) \leq f(t)(1 - a(t))^2 + b(t)
\end{equation}
for $a(t) := 2\eta_t \frac{1}{N} \lambda_{\text{min}, V_\parallel}(\E[X_tX_t^\cT])$ and
$b(t) := \Var[\norm{\eta_t\nabla_W \cL (W_t)}_F]$ so that $\tau(t) = \sum_{s = 0}^{t - 1} a(s)$.
Iterating the recursion (\ref{eq:int-rec}) shows that
\begin{align*}
  f(t) &\leq f(0) \prod_{s = 0}^{t - 1} (1 - a(s))^2 + \sum_{s = 0}^{t - 1} b(s) \prod_{r = s + 1}^{t - 1} (1 - a(r))^2\\
  &\leq e^{-2 \tau(t)}f(0) + \sum_{s = 0}^{t - 1} \frac{b(s)}{a(s)} e^{2 \tau(s + 1) - 2 \tau(t)} (\tau(s + 1) - \tau(s)).
\end{align*}
Changing variables to $\tau := \tau(t)$ and defining $A(\tau)$, $B(\tau)$, and $F(\tau)$ by $A(\tau(t)) = a(t)$, $B(\tau(t)) = b(t)$, and $F(\tau(t)) = f(t)$, we find
by replacing a right Riemann sum by an integral that
\begin{equation} \label{eq:int-time-int}
F(\tau) \precsim e^{-2\tau} \left[F(0)
+ \int_0^{\tau} \frac{B(\sigma)}{A(\sigma)} e^{2\sigma} d\sigma\right].
\end{equation}
In order for the result of optimization to be independent of the starting point,
by (\ref{eq:int-time-int}) we must have $\tau \to \infty$ to remove the dependence
on $F(0)$; this provides one explanation for the appearance of $\tau$ in condition
(\ref{eq:mr1-main}).  Further, (\ref{eq:int-time-int}) implies that the fluctuations
at an intrinsic time are bounded by an integral against the function
$\frac{B(\sigma)}{A(\sigma)}$ which depends only on the ratio of $A(\sigma)$ and $B(\sigma)$.

In the case of minibatch SGD, our proof of Theorem \ref{thm:sgd} shows the intrinsic time is
$\tau(t) = \sum_{s = 0}^{t - 1} 2\eta_s \frac{1}{N} \lambda_{\text{min}, V_\parallel}(XX^\cT)$
and the ratio $\frac{b(t)}{a(t)}$ in (\ref{eq:int-time-int}) is by (\ref{eq:sgd-c-bound})
bounded uniformly by $\frac{b(t)}{a(t)} \leq C \cdot \frac{\eta_t}{B_t}$ for a constant $C > 0$.
Thus, keeping $\frac{b(t)}{a(t)}$ fixed as a function of $\tau$ suggests the ``linear scaling''
 $\eta_t \propto B_t$ used empirically in \cite{goyal2017accurate} and proposed via
an heuristic SDE limit in \cite{smith2018don}.

\section{Analysis of Noising Augmentations} \label{sec:noise-analysis}

In this section, we give a full analysis of the noising augmentations presented in
Section \ref{sec:add-noise}. Let us briefly recall the notation. As before, we
consider overparameterized linear regression with loss
\[
\cL(W;\mathcal D) = \frac{1}{N}\norm{WX-Y}_F^2,
\]
where the dataset $\mathcal D$ of size $N$ consists of data matrices $X,Y$ that
each have $N$ columns $x_i\in \R^n,y_i\in \R^p$ with $n >N.$ We optimize
$\mathcal L(W;\mathcal D)$ by augmented gradient descent or augmented stochastic gradient
descent with additive Gaussian noise.  This means that at each time $t$ we
replace $\mathcal D = (X,Y)$ by a random batch $\mathcal D_t = (X_t,Y)$ of
size $B_t$, where the columns $x_{i,t}$ of $X_t$ are
\[
x_{i,t} = x_i + \sigma_t G_{i, t},\qquad G_{i, t}\sim \mathcal N(0,1)\text{   i.i.d.}
\]
In the case of gradient descent, the batch consists of the entire dataset, and
the resulting data matrices are
\[
X_t = X + \sigma_t G_t \qquad \text{and} \qquad Y_t = Y.
\]
In the case of stochastic gradient descent, the batch consists of $B_t$ datapoints
chosen uniformly at random with replacement, and the resulting data matrices are
\[
X_t = c_t (X A_t + \sigma_t G_t) \qquad \text{and} \qquad Y_t = c_t Y A_t,
\]
where $c_t = \sqrt{N/B_t}$, $A_t \in \RR^{N \times B_t}$ has i.i.d. columns with
a single non-zero entry equal to $1$, and $G_t \in \RR^{n \times B_t}$ has i.i.d.
Gaussian entries.  In both cases, the proxy loss is
\[
\oL_t(W) := \frac{1}{N} \|Y - W X\|_F^2 + \sigma_t^2 \|W\|_F^2,
\]
which has ridge minimizer
\[
W_t^* = Y X^\cT (XX^\cT + \sigma_t^2 N \cdot \Id_{n \times n})^{-1},
\]
and the space $V_\parallel := \text{ column span of $\E[X_t X_t^\cT]$}$ is all of $\R^n$.
We now separately analyze the cases of GD and SGD in Theorems \ref{thm:gauss-gd}
and Theorem \ref{thm:gauss-sgd}, respectively.

\subsection{Proof of Theorem \ref{thm:gauss-gd} for GD} \label{sec:gauss-noise-analysis}

We begin by proving convergence without rates. For this, we seek to show that if $\sigma_t^2,\eta_t \to 0$
with $\sigma_t^2$ non-increasing and
  \begin{equation} \label{eq:gauss-mr-app}
    \sum_{t = 0}^\infty \eta_t \sigma_t^2 = \infty \qquad \text{and} \qquad
    \sum_{t = 0}^\infty \eta_t^2 \sigma_t^2 < \infty,
  \end{equation}
then, $W_t \overset{p} \to W_{\text{min}}$. 
We will do this by applying Theorem \ref{thm:mr-main}, so we check that our assumptions
imply the hypotheses of these theorems. For Theorem \ref{thm:mr-main}, we directly compute
\[
\E[Y_t X_t^\cT] = Y X^\cT \qquad \text{and} \qquad \E[X_t X_t^\cT] = X X^\cT + \sigma_t^2 N \cdot \Id_{n \times n}
\]
and 
\begin{align*}
  \E[X_t X_t^\cT X_t] &= XX^\cT X + \sigma_t^2 (N + n + 1) X \\
  \E[X_t X_t^\cT X_t X_t^\cT] &= X X^\cT X X^\cT + \sigma_t^2 \Big( (2N + n + 2) X X^\cT + \Tr(X X^\cT) \Id_{n \times n} \Big)
  + \sigma_t^4 N(N + n + 1) \Id_{n \times n}.
\end{align*}
We also find that 
\begin{align*}
  \|\Xi_t^*\|_F &= |\sigma_t^2 - \sigma_{t + 1}^2| N \left\|Y X^\cT \Big(X X^\cT + \sigma_t^2 N \cdot \Id_{n \times n}\Big)^{-1}
  \Big(X X^\cT + \sigma_{t + 1}^2 N \cdot \Id_{n \times n}\Big)^{-1}\right\|_F\\
  &\leq |\sigma_t^2 - \sigma_{t + 1}^2| N \|Y X^\cT [(X X^\cT)^+]^2 \|_F.
\end{align*}
Thus, because $\sigma_t^2$ is decreasing, we see that the hypothesis (\ref{eq:mr2-main}) of Theorem \ref{thm:mr-main} indeed holds. Further, we note that
\begin{multline*}
  \sum_{t = 0}^\infty \eta_t^2 \E\left[ \|X_t X_t^\cT - \E[X_t X_t^\cT]\|_F^2 + \|Y_t X_t^\cT - \E[Y_t X_t^\cT]\|_F^2\right]\\
  = \sum_{t = 0}^\infty \eta_t^2 \sigma_t^2 \Big(2(n + 1) \|X\|_F^2 + N \|Y\|_F^2 + \sigma_t^2 N n(n + 1)\Big)= O\lrr{\sum_{t = 0}^\infty \eta_t^2 \sigma_t^2},
\end{multline*}
which by \eqref{eq:gauss-mr-app} implies (\ref{eq:mr3-app}). Theorem \ref{thm:mr-main} and the fact that
$\lim_{t \to \infty} W_t^* = W_{\text{min}}$ therefore yield that $W_t \overset{p} \to W_{\text{min}}$.

We now prove convergence with rates, we aim to show that if $\eta_t = \Theta(t^{-x})$ and $\sigma_t^2 = \Theta(t^{-y})$
with $x, y > 0$, $x + y < 1$, and $2x + y > 1$, then for any $\eps > 0$, we have that
\[
t^{\min\{y, \frac{1}{2}x\} - \eps} \|W_t - W_{\text{min}}\|_F \overset{p} \to 0.
\]
We now check the hypotheses for and apply Theorem \ref{thm:mr-quant}. For (\ref{eq:alpha2}),
notice that $Y_r = \Id - 2\eta_r \frac{X_r X_r^\cT}{N}$ satisfies the hypotheses of Theorem \ref{thm:mat-prod}
with $a_r = 1 - 2\eta_r \sigma_r^2$ and $b_r^2 = \frac{\eta_r^2 \sigma_r^2}{a_r^2} \Big(2 (n + 1) \|X\|_F^2 + \sigma_r^2 N n (n + 1)\Big)$.
Thus, by Theorem \ref{thm:mat-prod} and the fact that $\eta_t = \Theta(t^{-x})$ and $\sigma_t^2 = \Theta(t^{-y})$,
we find for some $C_1, C_2 > 0$ that
\[
  \log \E\left[\left\|\prod_{r = s}^t (\Id - 2\eta_r \frac{X_r X_r^\cT}{N})\right\|^2_2\right]
  \leq \sum_{r = s}^t b_r^2 + 2 \sum_{r = s}^t \log(1 - 2 \eta_r \sigma_r^2)
  \leq C_1 - C_2 \int_s^{t + 1} r^{- x - y} dr.
\]
For (\ref{eq:beta1}), we find that
\[
\|\Xi_t^*\|_F \leq |\sigma_t^2 - \sigma_{t + 1}^2| N \|Y X^\cT [(X X^\cT)^+]^2 \|_F = O(t^{-y - 1}).
\]
Finally, for (\ref{eq:beta2}), we find that
\begin{align*}
\eta_t^2 &\Tr\left[\Id \circ \Var\Big(\E[W_t] X_t X_t^\cT - Y_t X_t^\cT\Big)\right]
= \eta_t^2 \E\Big[\|\E[W_t](X_t X_t^\cT - \E[X_t X_t^\cT]) - (Y_t X_t^\cT - \E[Y_t X_t^\cT])\|_F^2\Big]\\
&\leq 2 \eta_t^2 \E\Big[\|\E[W_t](X_t X_t^\cT - \E[X_t X_t^\cT])\|_F^2 + \|Y_t X_t^\cT - \E[Y_t X_t^\cT]\|_F^2\Big]\\
&\leq 2 \eta_t^2 \Big(\|\E[W_t]\|_F^2 \E[\|X_t X_t^\cT - \E[X_t X_t^\cT]\|_F^2] + \|Y_t X_t^\cT - \E[Y_t X_t^\cT]\|_F^2\Big)\\
&= O(t^{-2x - y}).
\end{align*}
Noting finally that $\|W_t^* - W_{\text{min}}\|_F = O(\sigma_t^2) = O(t^{-y})$, we apply Theorem \ref{thm:mr-quant}
with $\alpha = x + y$, $\beta_1 = y + 1$, and $\beta_2 = 2x + y$ to obtain the desired estimates.
This concludes the proof of Theorem \ref{thm:gauss-gd}. \hfill $\square$

\subsection{Proof of Theorem \ref{thm:gauss-sgd} for SGD} \label{sec:gauss-noise-sgd-analysis}

We now prove Theorem \ref{thm:gauss-sgd} for SGD.  As before, we will apply
Theorems \ref{thm:mr-main} and \ref{thm:mr-quant}.  To check the hypotheses of
these theorems, we will use the following expressions for moments of the augmented data matrices.

\begin{lemma}\label{lem:noise-moments}
We have
\begin{equation}\label{eq:second-moments-gauss-sgd}
\E[Y_t X_t^\cT] = Y X^\cT \qquad \text{and} \qquad \E[X_t X_t^\cT] = XX^\cT + \sigma_t^2 N \Id_{n \times n}.
\end{equation}
Moreover, 
\begin{align*}
  \E[Y_t X_t^\cT X_t Y_t^\cT] &= c_t^4 \E[Y A_t A_t^\cT X^\cT X A_t A_t^\cT Y^\cT + \sigma_t^2 Y A_t G_t^\cT G_t A_t^\cT Y^\cT]\\
  &= \frac{N}{B_t} Y \diag(X^\cT X) Y^\cT + \frac{B_t - 1}{B_t} Y X^\cT X Y^\cT + \sigma_t^2 N Y Y^\cT \\
  \E[Y_t X_t^\cT X_t X_t^\cT] &= c_t^4\E[ Y A_t A_t^\cT X^\cT X A_t A_t^\cT X^\cT + \sigma_t^2 Y A_t G_t^\cT G_t A_t^\cT X^\cT\\
    &\phantom{==} + \sigma_t^2 Y A_t G_t^\cT X A_t G_t^\cT + \sigma_t^2 Y A_t A_t^\cT X^\cT G_t G_t^\cT]\\
  &= \frac{N}{B_t} Y \diag(X^\cT X) X^\cT + \frac{B_t - 1}{B_t} Y X^\cT X X^\cT + \sigma_t^2 (N + \frac{n + 1}{B_t / N}) Y X^\cT\\
  \E[X_t X_t^\cT X_t X_t^\cT] &= c_t^4 \E[X A_t A_t^\cT X^\cT X A_t A_t^\cT X^\cT + \sigma_t^2 G_t G_t^\cT X A_t A_t^\cT X^\cT
    + \sigma_t^2 X A_t G_t^\cT G_t A_t^\cT X^\cT\\
    &\phantom{==} + \sigma_t^2 X A_t A_t^\cT X^\cT G_t G_t^\cT + \sigma_t^2 G_t A_t^\cT X^\cT G_t A_t^\cT X^\cT
    + \sigma_t^2 X A_t G_t^\cT X A_t G_t^\cT\\
    &\phantom{==} + \sigma_t^2 G_t A_t^\cT X^\cT X A_t G_t^\cT + \sigma_t^4 G_t G_t^\cT G_t G_t^\cT]\\
  &= \frac{N}{B_t} X \diag(X^\cT X) X^\cT + \frac{B_t - 1}{B_t} X X^\cT X X^\cT + \sigma_t^2 (2N + \frac{n + 2}{B_t /N}) X X^\cT\\
  &\phantom{==} + \sigma_t^2 \frac{N}{B_t} \Tr(XX^\cT) \Id_{n \times n} + \sigma_t^4 N(N + \frac{n + 1}{B_t / N}) \Id_{n \times n}.
\end{align*}
\end{lemma}
\begin{proof}
All these formulas are obtained by direct, if slightly tedious, computation. 
\end{proof}

With these expressions in hand, we now check the conditions of Theorem \ref{thm:mr-main}.
First, by the Sherman-Morrison-Woodbury matrix inversion formula, we have
\begin{align} \label{eq:gauss-sgd-xi-diff}
\|\Xi_t^*\|_F &= |\sigma_t^2N - \sigma_{t + 1}^2 N| \left\| YX^\cT(XX^\cT + \sigma_t^2 N \cdot \Id_{n \times n})^{-1}
(XX^\cT + \sigma_{t + 1}^2 N \cdot \Id_{n \times n})^{-1}\right\|_F\\ \nonumber
&\leq N |\sigma_t^2  - \sigma_{t + 1}^2| \left\|Y X^\cT [(XX^\cT)^+]^2\right\|_F.
\end{align}
Because $\sigma_t^2$ is non-increasing, this implies (\ref{eq:mr2-main}).  Next,
by Lemma \ref{lem:noise-moments} we have that
\[
\sum_{t = 0}^\infty \eta_t \lambda_{\text{min}, V_\parallel}(\E[X_t X_t^\cT]) \geq \sum_{t = 0}^\infty \eta_t \sigma_t^2 = \infty
\]
by the first given condition in (\ref{eq:sgd-noise-mr}), which verifies (\ref{eq:mr1-main}).
Finally, by Lemma \ref{lem:noise-moments}, we may compute
\begin{multline*}
\Ee{\norm{X_tX_t^\cT - \Ee{X_tX_t^\cT}}_F^2} \\ = \frac{1}{B_t} \Tr\Big(X (N \diag(X^\cT X) - X^\cT X) X^\cT\Big)
  + 2 \sigma_t^2\frac{n + 1}{B_t / N} \Tr(X X^\cT)  + \sigma_t^4 \frac{N n (n + 1)}{B_t / N}
\end{multline*}
and
\[
\E\left[\|Y_t X_t^\cT - \E[Y_t X_t^\cT]\|_F^2\right]
= \frac{1}{B_t} \Tr\Big(Y (N \diag(X^\cT X) - X^\cT X) Y^\cT\Big) + \sigma_t^2 N \Tr(Y Y^\cT).
\]
Together, these imply that for some constant $C > 0$, we have that
\[
\sum_{t = 0}^\infty \eta_t^2 \left[ \Ee{\norm{X_t X_t^\cT - \Ee{X_t X_t^\cT}}_F^2} + \Ee{\norm{Y_t X_t^\cT - \E[Y_t X_t^\cT]}_F^2}\right]
\leq \sum_{t = 0}^\infty C \eta_t^2 < \infty
\]
by the second given condition in (\ref{eq:sgd-noise-mr}), which verifies (\ref{eq:mr3-main}).
Thus the conditions of Theorem \ref{thm:mr-main} apply, which shows that $W_t \overset{p} \to W_\text{min}$,
as desired.

For rates of convergence, we check the conditions of Theorem \ref{thm:mr-quant}. For (\ref{eq:alpha2}),
we will apply Theorem \ref{thm:mat-prod} to bound $\log \E\norm{\prod_{r=s}^t \lrr{\Id - \frac{2\eta_r}{N}X_rX_r^\cT}}_2^2$.
By the second moment expression $\E[X_rX_r^\cT] = XX^\cT + \sigma_r^2 N \Id_{n \times n}$, we find that
\[
\norm{\Ee{\Id - \frac{2\eta_r}{N}X_rX_r^\cT}}_2 = 1 - 2\eta_r\sigma_r^2.
\]
Moreover, by Lemma \ref{lem:noise-moments}, for some constant $C > 0$ we have
\begin{align*}
  &  \Ee{\norm{\Id - \frac{2\eta_r}{N}X_rX_r^\cT - \Ee{\Id - \frac{2\eta_r}{N}X_rX_r^\cT}}_2^2}\\
  &=\frac{4\eta_r^2}{N^2}\Ee{\norm{X_rX_r^\cT - \Ee{X_rX_r^\cT}}_2^2}\\
  &\leq  \frac{4\eta_r^2}{N^2}\left[\frac{1}{B_t} \Tr\Big(X (N \diag(X^\cT X) - X^\cT X) X^\cT\Big)
      + 2 \sigma_t^2\frac{n + 1}{B_t / N} \Tr(X X^\cT)  + \sigma_t^4 \frac{N n (n + 1)}{B_t / N}\right]\\
  &  \leq C \eta_r^2.
\end{align*}
Applying Theorem \ref{thm:mat-prod} with $a_r = 1 - 2 \eta_r \sigma_r^2$ and $b_r^2 = \frac{C\eta_r^2}{a_r^2}$,
we find that
\begin{align*}
\log \E\norm{\prod_{r=s}^t \lrr{\Id - \frac{2\eta_r}{N}X_rX_r^\cT}}_2^2
&\leq \sum_{r = s}^t \frac{C\eta_r^2}{a_r^2} + 2 \log\lrr{\prod_{r=s}^t \lrr{1 - 2\eta_r\sigma_r^2}}\\
&\leq \sum_{r = s}^t \frac{C\eta_r^2}{a_r^2} - 4 \sum_{r=s}^t \eta_r\sigma_r^2.
\end{align*}
Because $\eta_r = \Theta(r^{-x})$ and $\sigma_r^2 = \Theta(r^{-y})$, we obtain (\ref{eq:alpha2})
with $\alpha = x + y$, $C_1 = \sum_{r = 0}^\infty \frac{C\eta_r^2}{a_r^2}$, and some $C_2 > 0$,
where $C_1$ is finite because $x > \frac{1}{2}$.  Next, (\ref{eq:beta1}) holds for $\beta_1 = - y - 1$
by (\ref{eq:gauss-sgd-xi-diff}).  Finally, it remains to bound
\[
\eta_t^2 \Tr\left[\Id \circ \Var(\E[W_t] X_t X_t^\cT - Y_t X_t^\cT\Big)\right]
\]
to verify (\ref{eq:beta2}).  Again using \ref{lem:noise-moments} and noting that $W_{\text{min}} X = Y$,
we find
\begin{align*}
  \eta_t^2 &\Tr\left[\Id \circ \Var(\E[W_t] X_t X_t^\cT - Y_t X_t^\cT)\right]\\
  &= \eta_t^2 \Tr\Big(\frac{1}{B_t} \E[W_t] X (N\diag(X^\cT X) - X^\cT X) X^\cT \E[W_t]^\cT\\
  &\phantom{==} + 2\sigma_t^2 \frac{n + 1}{B_t / N} \E[W_t] XX^\cT \E[W_t]^\cT + (\sigma_t^2\frac{N}{B_t} \Tr(XX^\cT) + \sigma_t^4 N \frac{n + 1}{B_t / N}) \E[W_t] \E[W_t]^\cT\Big)\\
  &\phantom{==} - 2 \eta_t^2 \Tr\Big(\frac{1}{B_t} Y(N \diag(X^\cT X) - X^\cT X) X^\cT \E[W_t]^\cT + \sigma_t^2 \frac{n + 1}{B_t / N} Y X^\cT \E[W_t]^\cT\Big)\\
  &\phantom{==} + \eta_t^2 \Tr\Big(\frac{1}{B_t} Y(N \diag(X^\cT X) - X^\cT X) Y^\cT + \sigma_t^2 N Y Y^\cT\Big)\\
  &\leq C \eta_t^2 ( \sigma_t^2 + \norm{\Delta_t}_F^2)
\end{align*}
for some $C > 0$ and $\Delta_t := \E[W_t] - W_\text{min}$.  Define also $\Delta_t' := \E[W_t] - W_t^*$
so that, exactly as in Proposition \ref{prop:mr-moments}, we have
\[
\Delta_{t+1}' = \Delta_t'\lrr{\Id - \frac{2\eta_t}{N}\Ee{X_tX_t^\cT}} + \frac{2}{N}\Xi_t^*,\qquad \Delta_{t}':=\Ee{W_t-W_t^*}.
\]
Since $\norm{\Xi_t^*}_F = O(t^{-y-1})$ and we already saw that
\[
\norm{\Id - \frac{2\eta_t}{N}\Ee{X_tX_t^\cT}}_2 = 1- 2\eta_t\sigma_t^2,
\]
we may use the single sided decay estimates of Lemma \ref{lem:rec-matrix-rate} to conclude that
$\norm{\Delta_t'}_F = O(t^{x-1})$.  This implies that 
\[
\|\Delta_t\|_F \leq \|\Delta_t'\|_F + \norm{W_t^*-W_{\min}}_F = O(t^{x-1}) + \Theta(t^{-y}) = \Theta(t^{-y})
\]
since we assumed that $x+y < 1$. Therefore, we obtain
\begin{align*}
  \eta_t^2 &\Tr\left[\Id \circ \Var(\E[W_t] X_t X_t^\cT - Y_t X_t^\cT)\right]\leq C \eta_t^2  (\sigma_t^2 + \Theta(t^{-2y}))  = \Theta(t^{-2x-y}),
\end{align*}
showing that condition \eqref{eq:beta2} holds with $\beta_2 = 2x+y$.  We have thus verified all
of the conditions of Theorem \ref{thm:mr-quant}, whose application completes the proof. \hfill $\square$

\section{Analysis of random projection augmentations} \label{sec:proj-analysis}

In this section, we give a full analysis of the random projection augmentations
presented in Section \ref{sec:projections}.  In this setting, we have a preconditioning
matrix $\Pi_t = Q_t \wPi_t Q_t^\cT \in \RR^{n \times n}$, where
$\wPi_t$ is a projection matrix and $Q_t$ are Haar random orthogonal matrices.
Define the normalized trace of the projection matrix by 
\begin{equation} \label{eq:norm-trace}
\gamma_{t} := \frac{\Tr(\Pi_t)}{n}.
\end{equation}
We consider the augmentation given by
\[
X_t = \Pi_t X \qquad \text{and} \qquad Y_t = Y.
\]
In this setting, we first record the values of the lower order moments of the augmented
data matrices, with proofs deferred to Section \ref{sec:moment-proof}

\begin{lemma} \label{lem:data-moments}
We have that
\begin{align} \label{eq:mom1}
  \E[X_t] &= \gamma_t X\\ \label{eq:mom2}
  \E[Y_t X_t^\cT] &= \gamma_t Y X^\cT\\ \label{eq:mom3}
  \E[X_t X_t^\cT] &= \frac{\gamma_t (\gamma_t + 1/n - 2/n^2)}{1 + 1/n - 2/n^2} X X^\cT + \frac{\gamma_t (1 - \gamma_t)}{n (1 + 1/n - 2/n^2)} \|X\|_F^2 \Id\\ \label{eq:mom4}
  \E[X_t X_t^\cT X_t] &= \frac{\gamma_t^2 + \gamma_t (1/n - 2/n^2)}{1 + 1/n - 2/n^2} XX^\cT X + \frac{\gamma_t (1 - \gamma_t)}{1 + 1/n - 2/n^2} \frac{\|X\|_F^2}{n} X\\ \label{eq:mom5}
  \E[\|X_t X_t^\cT\|_F^2] &= \frac{(\gamma_t^2 + \gamma_t (1/n - 2/n^2)) \|XX^\cT\|_F^2 + \gamma_t (1 - \gamma_t) \frac{1}{n} \|X\|_F^4}{1 + 1/n - 2/n^2}.
\end{align}
\end{lemma}
We now compute the proxy loss and its optima as follows.

\begin{lemma} \label{lem:proj-proxy-loss}
The proxy loss for randomly rotated projections with normalized trace $\gamma_t$ given by
(\ref{eq:norm-trace}) is
\[
\oL_t(W) = \frac{1}{N}\|Y - \gamma_t WX\|_F^2 + \frac{1}{N} \gamma_t (1 - \gamma_t) \|X\|_F^2 \|W\|_F^2
+ \frac{\gamma_t(1 - \gamma_t)}{N} \frac{1/n - 2/n^2}{1 + 1/n - 2/n^2} (\|WX\|_F^2 + \frac{1}{n} \|X\|_F^2 \|W\|_F^2).
\]
It has proxy optima
\[
W_t^* = YX^\cT \Big(\frac{\gamma_t + 1/n - 2/n^2}{1 + 1/n - 2/n^2} XX^\cT + \frac{1 - \gamma_t}{1 + 1/n - 2/n^2} \frac{\|X\|_F^2}{n}\Id\Big)^{-1}.
\]
\end{lemma}
\begin{proof}
Applying (\ref{eq:proxy-opt}), we find that
\[
\oL_t(W) = \frac{1}{N} \E[\|Y_t - W X_t\|_F^2] = \frac{1}{N} \E[\|Y_t\|_F^2] - \frac{2}{N}\E[\Tr(X_t^\cT W^\cT Y_t)] + \frac{1}{N} \E[\Tr(W^\cT W X_t X_t^\cT)].
\]
Applying Lemma \ref{lem:data-moments}, we conclude that
\begin{align*}
\oL_t(W) &= \frac{1}{N} \|Y\|_F^2 - \frac{2 \gamma_t}{N} \Tr(X W^\cT Y) + \frac{1}{N} \frac{\gamma_t (\gamma_t + 1/n - 2/n^2)}{1 + 1/n - 2/n^2} \|WX\|_F^2 + \frac{1}{N} \frac{\gamma_t (1 - \gamma_t)}{n (1 + 1/n - 2/n^2)} \|X\|_F^2 \|W\|_F^2\\
&= \frac{1}{N}\|Y - \gamma_t WX\|_F^2 + \frac{1}{N} \gamma_t (1 - \gamma_t) \|X\|_F^2 \|W\|_F^2
+ \frac{\gamma_t(1 - \gamma_t)}{N} \frac{1/n - 2/n^2}{1 + 1/n - 2/n^2} (\|WX\|_F^2 + \frac{1}{n} \|X\|_F^2 \|W\|_F^2).
\end{align*}
The formula for $W_t^*$ then results from Lemma \ref{lem:data-moments} and the general
formula (\ref{eq:min-norm-opt}).
\end{proof}

\subsection{Proof of Lemma \ref{lem:data-moments}} \label{sec:moment-proof}

We apply the Weingarten calculus \citep{collins2006integration} to compute integrals
of polynomial functions of the matrix entries of a Haar orthogonal matrix. Because
each matrix entry of the expectations in Lemma \ref{lem:data-moments} is such a
polynomial, this will allow us to compute the relevant expectations.  The main result
we use is the following on polynomials of degree at most $4$ and its corollary.

\begin{prop}[Corollary 3.4 of \cite{collins2006integration}] \label{prop:wein}
For an $n \times n$ orthogonal matrix $Q$ drawn from the Haar measure,
we have that
\begin{align} \label{eq:wein1}
\E[Q_{i_1 j_1} Q_{i_{2} j_{2}}] &= \frac{1}{n} \delta_{i_1 i_2} \delta_{j_1 j_2} \\ \nonumber
\E[Q_{i_1 j_1} Q_{i_{2} j_{2}} Q_{i_{3} j_{3}} Q_{i_{4} j_{4}}] &= \frac{-1}{n(n - 1)(n + 2)} \Big(\delta_{i_1 i_2} \delta_{i_3 i_4} + \delta_{i_1 i_3} \delta_{i_2 i_4} + \delta_{i_1 i_4} \delta_{i_2 i_3}\Big)
  \Big(\delta_{j_1 j_2} \delta_{j_3 j_4} + \delta_{j_1 j_3} \delta_{j_2 j_4} + \delta_{j_1 j_4} \delta_{j_2 j_3}\Big)\\ \label{eq:wein2}
  &\phantom{=} + \frac{1}{n(n - 1)} \Big(\delta_{i_1 i_2} \delta_{i_3 i_4} \delta_{j_1 j_2} \delta_{j_3 j_4}
  + \delta_{i_1 i_3} \delta_{i_2 i_4} \delta_{j_1 j_3} \delta_{j_2 j_4} + \delta_{i_1 i_4} \delta_{i_2 i_3} \delta_{j_1 j_4} \delta_{j_2 j_3} \Big),
\end{align}
where $\delta_{ab}$ denotes the Kronecker delta function $\delta_{ab} = \bI\{a = b\}$.
\end{prop}

\begin{corr} \label{corr:matrix-exp}
For matrices $A, B, C \in \RR^{n \times n}$ and an $n \times n$ Haar random orthogonal matrix $Q$,
we have
\begin{align} \label{eq:matrix-exp1}
\E[ Q A Q^\cT] &= \frac{1}{n} \Tr(A) \cdot \Id\\ \nonumber
\E[ Q A Q^\cT B Q C Q^\cT]_{ij} &= - \frac{1}{n(n - 1)(n + 2)} \Big(\Tr(A) \Tr(C) + \Tr(AC^\cT + AC)\Big)
  \Big(\delta_{ij} \Tr(B) + B_{ij} + B_{ji}\Big)\\ \label{eq:matrix-exp2}
  &\phantom{===}+ \frac{1}{n(n - 1)} \Big(\Tr(A) \Tr(C) B_{ij} + \Tr(AC^\cT) B_{ji} + \delta_{ij} \Tr(AC) \Tr(B)\Big).
\end{align}
\end{corr}
\begin{proof}
For (\ref{eq:matrix-exp1}), notice by (\ref{eq:wein1}) that 
\[
\E[ Q A Q^\cT ]_{ij} = \sum_{a, b = 1}^n \E[ Q_{ia} A_{ab} Q_{jb} ] = \sum_{a, b = 1}^n \delta_{ab}\delta_{ij} \frac{1}{n} A_{ab} = \delta_{ij} \frac{1}{n} \Tr(A).
\]
For (\ref{eq:matrix-exp2}), notice by (\ref{eq:wein2}) that
\begin{align*}
\E[Q A Q^\cT B Q C Q^\cT]_{ij} &= \sum_{a, b, c, d, e, f = 1}^n \E[Q_{ia} A_{ab} Q_{cb} B_{cd} Q_{de} C_{ef} Q_{jf}]\\
  &= -\frac{1}{n(n - 1)(n + 2)} \Big(\sum_{a, e = 1}^n A_{aa} C_{ee} + \sum_{a, b = 1}^n A_{ab} C_{ab} + \sum_{a, b = 1}^n A_{ab} C_{ba}\Big)
  \Big(\sum_{c = 1}^n \delta_{ij} B_{cc} + B_{ij} + B_{ji}\Big)\\
  &\phantom{==}+ \frac{1}{n(n - 1)} \Big(\sum_{a, e = 1}^n A_{aa} C_{ee} B_{ij} + \sum_{a, b = 1}^n A_{ab} C_{ab} B_{ji} + \delta_{ij} \sum_{a, b = 1}^n A_{ab} C_{ba} \sum_{c= 1}^n B_{cc}\Big),
\end{align*}
which when simplified gives the desired conclusion.
\end{proof}

We now compute each lower order moment; in each computation, let $Q$ denote a
random $n \times n$ orthogonal matrix drawn from the Haar measure. Claims (\ref{eq:mom1})
and (\ref{eq:mom2}) follow from Corollary \ref{corr:matrix-exp} and the facts that
\[
\E[X_t] = \E[Q \wPi_t Q^\cT] X \qquad \text{and} \qquad \E[Y_t X_t^\cT] = Y \E[ Q \wPi_t Q^\cT ] X.
\]
Claims (\ref{eq:mom3}) and (\ref{eq:mom4}) follow from Corollary \ref{corr:matrix-exp} and
the facts that
\[
\E[X_t X_t^\cT] = \E[Q \wPi_t Q^\cT XX^\cT Q \wPi_t Q^\cT] \qquad \text{and} \qquad
\E[X_t X_t^\cT X_t] = \E[ Q \wPi_t Q^\cT XX^\cT Q \wPi_t^2 Q^\cT] X.
\]
Finally, (\ref{eq:mom5}) follows from Corollary \ref{corr:matrix-exp} and the fact that
\[
\E[\|X_t X_t^\cT]_F^2] = \E[\Tr(Q \wPi_t Q^\cT XX^\cT Q \wPi_t^2 Q^\cT X X^\cT Q \wPi_t Q^\cT)]
= \E[\Tr(\wPi_t Q^\cT XX^\cT Q \wPi_t^2 Q^\cT X X^\cT Q \wPi_t)]].
\]
This completes the proof of Lemma \ref{lem:data-moments}.

\subsection{Proof of Theorem \ref{thm:proj-main}}

We first show convergence.  By Lemma \ref{lem:data-moments}, we find that
$V_\parallel = \RR^n$.  Furthermore, because $\gamma_t \to 1$, we have that
$W_\infty^* = \lim_{t \to \infty} W_t^* = W_{\min}$. It therefore suffices to
verify the conditions of Theorem \ref{thm:mr-main}. First, notice that
\begin{align*}
\|\Xi_t^*\|_F &= \frac{|\gamma_t - \gamma_{t + 1}|}{1 + 1/n - 2/n^2}  \left\| YX^\cT
\Big(\frac{\gamma_{t + 1} + 1/n - 2/n^2}{1 + 1/n - 2/n^2} XX^\cT + \frac{1 - \gamma_{t + 1}}{1 + 1/n - 2/n^2} \frac{\|X\|_F^2}{n} \Id\Big)^{-1}\right.\\
&\phantom{==========}\left. \Big(XX^\cT - \frac{\|X\|_F^2}{n} \Id\Big)
\Big(\frac{\gamma_t + 1/n - 2/n^2}{1 + 1/n - 2/n^2} XX^\cT + \frac{1 - \gamma_t}{1 + 1/n - 2/n^2} \frac{\|X\|_F^2}{n} \Id\Big)^{-1}\right\|_F\\
&\leq \frac{|\gamma_t^{-1} - \gamma_{t + 1}^{-1}|}{1 + 1/n - 2/n^2} \|YX^\cT\|_F \|[XX^\cT]^+\|^2_F \|XX^\cT - \frac{1}{n}\|X\|_F^2\Id \|_F.
\end{align*}
Because $\gamma_t$ are increasing and $\lim_{t \to \infty} \gamma_t = 1$, we find that
(\ref{eq:mr2-main}) holds.  Now, by Lemma \ref{lem:data-moments}, we have
\[
\lambda_{\min}(\E[X_t X_t^\cT]) \geq \frac{\gamma_t (1 - \gamma_t)}{1 + 1/n - 2/n^2} \frac{\|X\|_F^2}{n},
\]
so the first condition in (\ref{eq:proj-mr}) implies (\ref{eq:mr1-main}).
Finally, using Lemma \ref{lem:data-moments} we may compute
\begin{align} \nonumber
\E[\|X_t X_t^\cT &- \E[X_t X_t^\cT]\|_F^2] = \E[\|X_t X_t\|_F^2] - \|\E[X_t X_t^\cT]\|_F^2\\ \nonumber
&= \frac{\gamma_t (1 - \gamma_t) (\gamma_t(1 + \gamma_t) n^4 + (1 + 2 \gamma_t) n^3 - (1 + 4 \gamma_t) n^2 - 4 n + 4)}{(n - 1)^2 (n + 2)^2} \|XX^\cT\|_F^2\\ \nonumber
&\phantom{===}+ \frac{\gamma_t(1 - \gamma_t) ((1 - \gamma_t - \gamma_t^2) n^4 + (1 - 2 \gamma_t) n^3 + (4 \gamma_t - 2) n^2)}{(n - 1)^2 (n + 2)^2} \frac{1}{n} \|X\|_F^4\\ \label{eq:mom4-bound}
&\leq 2 \gamma_t (1 - \gamma_t) \Big(\|XX^\cT\|_F^2 + \frac{1}{n} \|X\|_F^4\Big)
\end{align}
and also
\begin{align} \nonumber
\E[\|Y_t X_t^\cT &- \E[Y_t X_t^\cT]\|_F^2] = \E[\|Y_t X_t^\cT\|_F^2] - \|\E[Y_t X_t^\cT]\|_F^2 \\ \nonumber
&= \frac{\gamma_t (1 - \gamma_t) (1/n - 2/n^2)}{1 + 1/n - 2/n^2} \|YX^\cT\|_F^2  + \frac{\gamma_t (1 - \gamma_t)}{1 + 1/n - 2/n^2}\frac{1}{n} \|X\|_F^2 \|Y\|_F^2\\ \label{eq:mom2-bound}
&\leq \gamma_t (1 - \gamma_t)\Big(\|YX^\cT\|_F^2 + \frac{1}{n} \|X\|_F^2 \|Y\|_F^2\Big).
\end{align}
Combining these bounds with the second condition in (\ref{eq:proj-mr}) gives
(\ref{eq:mr3-main}).  Having verified (\ref{eq:mr2-main}), (\ref{eq:mr1-main}),
and (\ref{eq:mr3-main}), we conclude by Theorem \ref{thm:mr-main} that
$W_t \overset{p} \to W_{\min}$ as desired.

We now obtain rates using Theorem \ref{thm:mr-quant}.  We aim to show that
if $\eta_t = \Theta(t^{-x})$ and $\gamma_t = 1 - \Theta(t^{-y})$ with $x, y > 0$,
$x + y < 1$, and $2x + y > 1$, then for any $\eps > 0$ we have that
\[
t^{\min\{y, \frac{1}{2}x\} - \eps}\|W_t - W_{\min}\|_F \overset{p} \to 0.
\]
We now check the hypotheses of Theorem \ref{thm:mr-quant}.  For (\ref{eq:alpha2}),
by Lemma \ref{lem:data-moments} and (\ref{eq:mom4-bound}), $Y_r = \Id - \frac{2\eta_r}{N} X_r X_r^\cT$
satisfies the hypotheses of Theorem \ref{thm:mat-prod} with
\[
a_r = 1 - \frac{2\eta_r}{N} \frac{\gamma_t (1 - \gamma_t)}{1 + 1/n - 2/n^2} \frac{\|X\|_F^2}{n} \qquad \text{and} \qquad
b_r^2 = \frac{8 \eta_r^2 \gamma_t (1 - \gamma_t) }{a_r^2 N^2} \Big(\|XX^\cT\|_F^2 + \frac{1}{n}\|X\|_F^4\Big).
\]
By Theorem \ref{thm:mat-prod} and the fact that $\eta_t = \Theta(t^{-x})$ and $\gamma_t = 1 - \Theta(t^{-y})$,
we find for some $C_1, C_2 > 0$ that
\begin{align*}
  \log \E\left[\left\|\prod_{r = s}^t (\Id - \frac{2\eta_r}{N} X_r X_r^\cT)\right\|^2_2\right]
  &\leq \sum_{r = s}^t b_r^2 + 2 \sum_{r = s}^t \log(1 - \frac{2 \eta_r}{N(1 + 1/n - 2/n^2)} \gamma_t (1 - \gamma_t) \frac{\|X\|_F^2}{n})\\
  &\leq C_1 - C_2 \int_s^{t + 1} r^{- x - y} dr.
\end{align*}
For (\ref{eq:beta1}), our previous computations show that 
\[
\|\Xi_t^*\|_F \leq \frac{|\gamma_t^{-1} - \gamma_{t + 1}^{-1}|}{1 + 1/n - 2/n^2} \|YX^\cT\|_F \|[XX^\cT]^+\|^2_F \|XX^\cT - \frac{1}{n}\|X\|_F^2\Id \|_F
= O(t^{-y - 1}).
\]
Finally, for (\ref{eq:beta2}), we find by (\ref{eq:mom4-bound}), (\ref{eq:mom2-bound}) and the fact that
$\|\E[W_t]\|_F = O(1)$ that 
\begin{align*}
\eta_t^2& \Tr\left[\Id \circ \Var\Big(\E[W_t] X_t X_t^\cT - Y_t X_t^\cT\Big)\right]
= \eta_t^2 \E\Big[\|\E[W_t](X_t X_t^\cT - \E[X_t X_t^\cT]) - (Y_t X_t^\cT - \E[Y_t X_t^\cT])\|_F^2\Big]\\
&\leq 2 \eta_t^2 \E\Big[\|\E[W_t](X_t X_t^\cT - \E[X_t X_t^\cT])\|_F^2 + \|Y_t X_t^\cT - \E[Y_t X_t^\cT]\|_F^2\Big]\\
&\leq 2 \eta_t^2 \Big(\|\E[W_t]\|_F^2 \E[\|X_t X_t^\cT - \E[X_t X_t^\cT]\|_F^2] + \|Y_t X_t^\cT - \E[Y_t X_t^\cT]\|_F^2\Big)\\
&= O(t^{-2x - y}).
\end{align*}
Noting finally that $\|W_t^* - W_{\text{min}}\|_F = O(1 - \gamma_t) = O(t^{-y})$, we apply Theorem \ref{thm:mr-quant}
with $\alpha = x + y$, $\beta_1 = y + 1$, and $\beta_2 = 2x + y$ to obtain the desired rate.
This concludes the proof of Theorem \ref{thm:proj-main}.\hfill $\square$

\subsection{Experimental validation} \label{sec:proj-exp}

To validate Theorem \ref{thm:proj-main}, we ran augmented GD with random projection
augmentation on $N = 100$ simulated datapoints.  Inputs were i.i.d. Gaussian
vectors in dimension $n = 400$, and outputs in dim $p = 1$ were generated by a random linear
map with i.i.d Gaussian coefficients drawn from $\cN(1, 1)$.  The learning rate followed a
 fixed polynomially decaying schedule
$\eta_t = \frac{0.005}{100} \cdot (\text{batch size}) \cdot (1 + \frac{t}{20})^{- 0.66}$.
Figure \ref{fig:proj-figure} shows MSE and $\|W_{t, \perp}\|_F$ along a single optimization trajectory with
different schedules for the dimension ratio $1 - \gamma_t$ used in random projection augmentation.
Code to generate this figure is provided in \url{supplement.zip} in the supplement. It
ran in 30 minutes on a standard laptop CPU.

\begin{figure}[ht!]
  \centering
  \setlength\tabcolsep{10pt} 
  \begin{tabular}{cc}
    \includegraphics[width=2.5in]{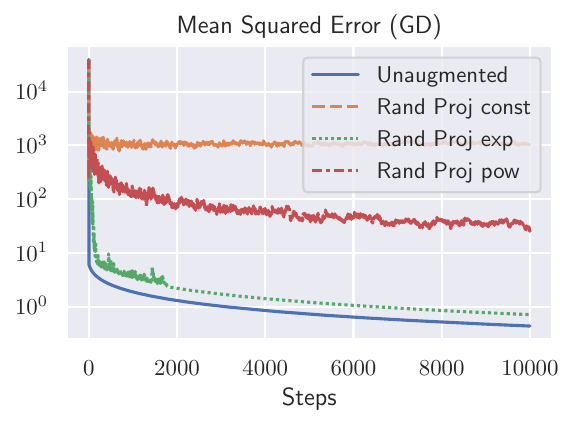} & \includegraphics[width=2.5in]{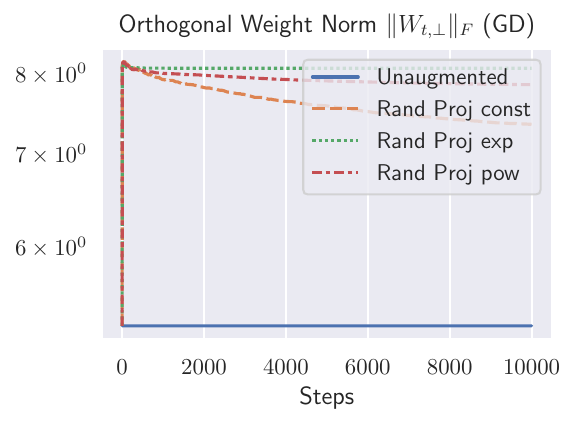}
  \end{tabular}
    \caption{MSE and $\|W_{t, \perp}\|_F$ for optimization trajectories of
    GD with random projection augmentation Jointly scheduling learning
    rate and noise variance to have polynomial decay is necessary
    for optimization to converge to the minimal norm solution $W_{\min}$.
    Rand Proj const, Rand Proj exp, and Rand Proj pow have random projection
    augmentation with $1 - \gamma_t = 0.1, 0.1 e^{-0.02t}, 0.1(1 + \frac{t}{50})^{- 0.33}$, respectively.
    Other experimental details are the same as those described in \S \ref{sec:proj-exp}.
  }
  \label{fig:proj-figure}
\end{figure}

\section{Analysis of SGD} \label{sec:sgd-analysis}

This section gives an analysis of vanilla SGD using our method.
Let us briefly recall the notation. As before, we consider overparameterized linear regression with loss
\[\mathcal L(W;\mathcal D) = \frac{1}{N}\norm{WX-Y}_F^2,\]
where the dataset $\mathcal D$ of size $N$ consists of data matrices $X,Y$ that each have $N$ columns $x_i\in \R^n,y_i\in \R^p$ with $n >N.$ We optimize $\mathcal L(W;\mathcal D)$ by augmented SGD either with or without additive Gaussian noise. In the former case, this means that at each time $t$ we replace $\mathcal D = (X,Y)$ by a random batch $\cB_t = (X_t,Y_t)$ given by a prescribed batch size $B_t = |\cB_t|$ in which each datapoint in $\cB_t$ is chosen uniformly with replacement
from $\cD$, and the resulting data matrices $X_t$ and $Y_t$ are scaled so that $\overline{\cL}_t(W) = \cL(W; \cD)$. Concretely, this means that for the normalizing factor $c_t := \sqrt{N / B_t}$ we have
\begin{equation}\label{eq:SGD-aug}
X_t = c_t X A_t \qquad \text{and} \qquad Y_t = c_t Y A_t,
\end{equation}
where $A_t \in \RR^{N \times B_t}$ has i.i.d. columns $A_{t, i}$ with a single non-zero entry equal to $1$ chosen uniformly at random. In this setting the minimum norm optimum for each $t$ are the same
and given by
\[
W_t^* = W_{\min} = Y X^\cT (X X^\cT)^+,
\]
which coincides with the minimum norm optimum for the unaugmented loss.

\begin{theorem} \label{thm:sgd}
  If the learning rate satisfies $\eta_t \to 0$ and
  \begin{equation} \label{eq:sgd-mr}
    \sum_{t = 0}^\infty \eta_t = \infty,
  \end{equation}
  then for any initialization $W_0$, we have $W_tQ_\parallel \overset{p} \to W_{\min}$. If further we
  have that $\eta_t = \Theta(t^{-x})$ with $0 < x < 1$, then for some $C > 0$ we have
  \[
  e^{C t^{1 - x}} \|W_t Q_\parallel - W_{\min}\|_F \overset{p} \to 0.
  \]
\end{theorem}

Theorem \ref{thm:sgd} recovers the exponential convergence rate for SGD in the
overparametrized settting, which has been previously studied through both empirical
and theoretical means \citep{ma2018power}.  Because $1 \leq B_t \leq N$
for all $t$, it does not affect the asymptotic results in Theorem \ref{thm:sgd}.
In practice, however, the number of optimization steps $t$ is often small enough
that $\frac{B_t}{N}$ is of order $t^{-\alpha}$ for some $\alpha > 0$, meaning the
choice of $B_t$ can affect rates in this non-asymptotic regime. Though we do not
attempt to push our generic analysis to this granularity, this is done in
\cite{ma2018power} to derive optimal batch sizes and learning rates in the
overparametrized setting.

\begin{proof}[Proof of Theorem \ref{thm:sgd}]
In order to apply Theorems \ref{thm:mr} and \ref{thm:mr-quant}, we begin by computing the moments
of $A_t$ as follows.  Recall the notation $\diag(M)$ from Appendix \ref{sec:notations}.

\begin{lemma} \label{lem:a-moments}
For any $Z \in \RR^{N \times N}$, we have that
\[
\E[A_t A_t^\cT] = \frac{B_t}{N} \Id_{N \times N} \qquad \text{and} \qquad
\E[A_t A_t^\cT Z A_t A_t^\cT] = \frac{B_t}{N} \diag(Z) + \frac{B_t (B_t - 1)}{N^2} Z. 
\]
\end{lemma}
\begin{proof}
We have that
\[
\E[A_t A_t^\cT] = \sum_{i = 1}^{B_t} \E[A_{i, t} A_{i, t}^\cT] = \frac{B_t}{N} \Id_{N \times N}. 
\]
Similarly, we find that
\begin{align*}
  \E[A_t A_t^\cT Z A_t A_t^\cT] &= \sum_{i, j = 1}^{B_t} \E[A_{i, t} A_{i, t}^\cT Z A_{j, t} A_{j, t}^\cT]\\
  &= \sum_{i = 1}^{B_t} \E[A_{i, t} A_{i, t}^\cT Z A_{i, t} A_{i, t}^\cT] + 2 \sum_{1 \leq i < j \leq B_t} \E[A_{i, t} A_{i, t}^\cT Z A_{j, t} A_{j, t}^\cT]\\
  &= \frac{B_t}{N} \diag(Z) + \frac{B_t (B_t - 1)}{N^2} Z,
\end{align*}
which completes the proof.
\end{proof}

Let us first check convergence in mean:
\[\E[W_{t}] Q_\parallel \to W_{\min}.\]
To see this, note that Lemma \ref{lem:a-moments} implies 
\[
\E[Y_t X_t^\cT] = Y X^\cT \qquad \E[X_t X_t^\cT] = XX^\cT,
\]
which yields that 
\begin{equation}\label{E:sgd-opt}
W_t^* = Y X^\cT [XX^\cT]^+ = W_{\min}
\end{equation}
for all $t$. We now prove convergence.  Since all $W_t^*$ are equal to $W_{\min}$, we find that $\Xi^*_t = 0$.
By (\ref{eq:exp-rec}) and Lemma \ref{lem:a-moments} we have
\[
\E[W_{t + 1}] - W_{\min} = (\E[W_t] - W_{\min}) \Big(\Id - \frac{2 \eta_t}{N} XX^\cT\Big),
\]
which implies since $\frac{2 \eta_t}{N} < \lambda_{\text{max}}(X X^\cT)^{-1}$ for large $t$
that for some $C > 0$ we have
\begin{multline} \label{eq:delta-bound}
  \|\E[W_t]Q_\parallel - W_{\min}\|_F \leq \|W_0 Q_\parallel - W_{\min}\|_F \prod_{s = 0}^{t - 1} \left\|Q_\parallel - \frac{2 \eta_s}{N} XX^\cT\right\|_2\\
  \leq C \|W_0 Q_\parallel - W_{\min}\|_F \exp\Big(- \sum_{s = 0}^{t - 1} \frac{2 \eta_s}{N} \lambda_{\text{min}, V_\parallel}(XX^\cT)\Big).
\end{multline}
From this we readily conclude using (\ref{eq:sgd-mr}) the desired convergence in mean $\E[W_{t}] Q_\parallel \to W_{\min}$.

Let us now prove that the variance tends to zero. By Proposition \ref{prop:mr-moments}, we find that
$Z_t = \E[(W_t - \E[W_t])^\cT (W_t - \E[W_t])]$ has two-sided decay of type $(\{A_t\}, \{C_t\})$
with 
\[A_t = \frac{2 \eta_t}{N} X_tX_t^\cT,\qquad C_t = \frac{4 \eta_t^2}{N^2} \left[\Id \circ \Var((\E[W_t] X_t  - Y_t) X_t^\cT)\right].\]
To understand the resulting rating of convergence, let us first obtain a bound on $\Tr(C_t)$. To do this, note that for any matrix $A$, we have
\[\tr\lrr{\Id \circ \Var[A]}= \tr\lrr{\Ee{A^\cT A}- \Ee{A}^\cT\Ee{A}}.\]
Moreover, using the definition \eqref{eq:SGD-aug} of the matrix $A_t$ and writing
\[M_t:=\Ee{W_t}X-Y,\]
 we find
\[\lrr{(\Ee{W_t} X_t  - Y_t) X_t^\cT}^\cT(\Ee{W_t} X_t  - Y_t) X_t^\cT = XA_tA_t^\cT M_t^\cT M_t A_tA_t^\cT X^\cT\]
as well as 
\[\Ee{\lrr{(\E[W_t] X_t  - Y_t) X_t^\cT}}^\cT\Ee{(\E[W_t] X_t  - Y_t) X_t^\cT} = X\Ee{A_tA_t^\cT} M_t^\cT M_t \Ee{A_tA_t^\cT} X^\cT.\]
Hence, using the expression from Lemma \ref{lem:a-moments} for the moments of $A_t$ and recalling the scaling factor $c_t=(N/B_t)^{1/2}$, we find 
\[\Tr(C_t)=\frac{4\eta_t^2}{B_t} \Tr\lrr{X\left\{\mathrm{diag}\lrr{M_t^\cT M_t} -\frac{1}{N}M_t^\cT M_t\right\}X^\cT}.\]
Next, writing 
\[\Delta_t := \E[W_t] - W_{\min}\]
and recalling \eqref{E:sgd-opt}, we see that
\[M_t = \Delta_t X.\]
Thus, applying the estimates (\ref{eq:delta-bound}) about exponential convergence of the mean,
we obtain
\begin{equation} \label{eq:sgd-c-bound}
  \Tr(C_t) \leq \frac{8 \eta_t^2}{B_t}\norm{\Delta_t Q_{||}}_2^2\norm{XX^T}_2^2
  \leq C\frac{8 \eta_t^2}{B_t}\norm{XX^T}_2^2\|\Delta_0 Q_\parallel\|_F^2 \exp\Big(- \sum_{s = 0}^{t - 1} \frac{4 \eta_s}{N} \lambda_{\text{min}, V_\parallel}(XX^\cT)\Big).
\end{equation}
Notice now that $Y_r = Q_\parallel - A_r$ satisfies the conditions of Theorem \ref{thm:mat-prod} with
$a_r = 1 - 2 \eta_r \frac{1}{N} \lambda_{\text{min}, V_\parallel}(XX^\cT)$ and
$b_r^2 = \frac{4\eta_r^2}{B_r a_r^2 N} \Tr\Big(X \diag(X^\cT X) X - \frac{1}{N} XX^\cT XX^\cT\Big)$.
By Theorem \ref{thm:mat-prod} we then obtain for any $t > s > 0$ that
\begin{equation} \label{eq:sgd-sq-norm-bound}
  \E\left[\left\|\prod_{r = s + 1}^t (Q_\parallel - A_r) \right\|_2^2 \right]
  \leq e^{\sum_{r = s + 1}^t b_r^2} \prod_{r = s + 1}^t \Big(1 - 2 \eta_r \frac{1}{N} \lambda_{\text{min}, V_\parallel}(XX^\cT)\Big)^2.
\end{equation}
By two-sided decay of $Z_t$, we find by (\ref{eq:sgd-c-bound}), (\ref{eq:sgd-sq-norm-bound}),
and (\ref{eq:two-side-bound}) that
\begin{multline} \label{eq:sec-moment-bound}
  \E[\|W_tQ_\parallel - \E[W_t]Q_\parallel\|^2_F] = \Tr(Q_\parallel Z_{t} Q_\parallel)\\
  \leq e^{- \frac{4}{N} \lambda_{\text{min}, V_\parallel}(XX^\cT) \sum_{s = 0}^{t - 1} \eta_s}
  \frac{\|XX^\cT\|_2^2}{N^2} \|\Delta_0Q_\parallel\|_F^2 C
  \sum_{s = 0}^{t - 1} \frac{8 \eta_s^2}{B_s/N} e^{\frac{4 \eta_s}{N} \lambda_{\text{min}, V_\parallel}(XX^\cT) + \sum_{r = s + 1}^t b_r^2}.
\end{multline}
Since $\eta_s \to 0$, we find that $\eta_s \frac{N}{B_s} e^{\frac{4 \eta_s}{N} \lambda_{\text{min}, V_\parallel}(XX^\cT)}$
is uniformly bounded and that $b_r^2 \leq \frac{4}{N} \lambda_{\text{min}, V_\parallel}(XX^\cT) \eta_r$
for sufficiently large $r$.  We therefore find that for some $C'>0$, 
\[
 \E[\|W_tQ_\parallel - \E[W_t]Q_\parallel\|_F^2] \leq C' \sum_{s = 0}^{t - 1} \eta_s e^{- \frac{4}{N} \lambda_{\text{min}, V_\parallel}(XX^\cT) \sum_{r = 0}^{s} \eta_r},
\]
hence  $\lim_{t \to \infty} \E[\|W_tQ_\parallel - \E[W_t]Q_\parallel\|_F^2] = 0$ by Lemma \ref{lem:sum-bound}.
Combined with the fact that $\E[W_t] Q_\parallel \to W_{\min}$, this implies that
$W_t Q_\parallel \overset{p} \to W_{\min}$.

To obtain a rate of convergence, observe that by (\ref{eq:delta-bound}) and the
fact that $\eta_t = \Theta(t^{-x})$, for some $C_1, C_2 > 0$ we have
\begin{equation} \label{eq:exp-bounds}
\|\E[W_{t}]Q_\parallel - W_{\min}\|_F \leq C_1 \exp\Big(- C_2 t^{1 - x}\Big).
\end{equation}
Similarly, by (\ref{eq:sec-moment-bound}) and the fact that $\frac{\eta_s}{B_s/N} < \infty$ uniformly,
for some $C_3, C_4, C_5 > 0$ we have
\[
\E[\|W_tQ_\parallel - \E[W_t]Q_\parallel\|_F^2] \leq C_3 \exp\Big(- C_4 t^{1 - x}\Big) t^{1 - x}
\]
We conclude by Chebyshev's inequality that for any $a > 0$ we have
\[
\PP\Big(\|W_t Q_\parallel - W_{\min}\|_F
\geq C_1 \exp\Big(- C_2 t^{1 - x}\Big) + a \cdot \sqrt{C_3} t^{\frac{1}{2} - \frac{x}{2}} e^{- C_4 t^{1 - x}/2}\Big)
\leq a^{-2}.
\]
Taking $a = t$, we conclude as desired that for some $C > 0$, we have
\[
e^{C t^{1 - x}} \|W_t Q_\parallel - W_{\min}\|_F \overset{p} \to 0.
\]
This completes the proof of Theorem \ref{thm:sgd}.
\end{proof}

\end{document}